\documentclass[10pt,twocolumn,letterpaper]{article}

\usepackage{cvpr}              

\usepackage{graphicx}
\usepackage{amsmath}
\usepackage{amssymb}
\usepackage{booktabs}

\usepackage{gensymb}
\usepackage[pagebackref,breaklinks,colorlinks]{hyperref}

\usepackage[capitalize]{cleveref}
\crefname{section}{Sec.}{Secs.}
\Crefname{section}{Section}{Sections}
\Crefname{table}{Table}{Tables}
\crefname{table}{Tab.}{Tabs.}

\usepackage[utf8]{inputenc} 
\usepackage[T1]{fontenc}    
\usepackage{hyperref}       

\usepackage{booktabs}       
\usepackage{amsfonts}       
\usepackage{nicefrac}       
\usepackage{microtype}      
\usepackage{xcolor}         
\usepackage{amsmath,mathtools}
\usepackage{amssymb}
\usepackage{amsthm}
\usepackage{graphicx}
\usepackage{enumitem}
\usepackage{bm}
\usepackage{bbm}
\usepackage{enumitem}
\usepackage{tabularx}
\usepackage{multirow}

\usepackage{tabularx}
\newcolumntype{L}{>{\raggedright\arraybackslash}X}
\newcolumntype{C}{>{\centering\arraybackslash}X}

\definecolor{mkcolor}{RGB}{255,0,128}

\definecolor{ricardocolor}{RGB}{255, 87, 51}

\definecolor{leocolor}{RGB}{0,0,255}

\definecolor{kecolor}{RGB}{0,128,0}

\def\cvprPaperID{3232} 
\def\confName{CVPR}
\def\confYear{2023}

\begin{document}

\title{\vspace{-0.4cm}SCADE: NeRFs from Space Carving with Ambiguity-Aware Depth Estimates\vspace{-0.4cm}}

\author{Mikaela Angelina Uy$^{1,2}$~~~Ricardo Martin-Brualla$^{2}$~~~Leonidas Guibas$^{1,2}$~~~Ke Li$^{2,3}$
\vspace{0.2cm}\\
$^1$Stanford University~~~~~$^2$Google~~~~~$^3$Simon Fraser University
\vspace{0.3cm}\\
}

\twocolumn[{
\renewcommand\twocolumn[1][]{#1}%
\maketitle
\vspace{-0.485in}
\begin{center}
\vspace{-5pt}
    \centering
    \includegraphics[width=0.93\textwidth]{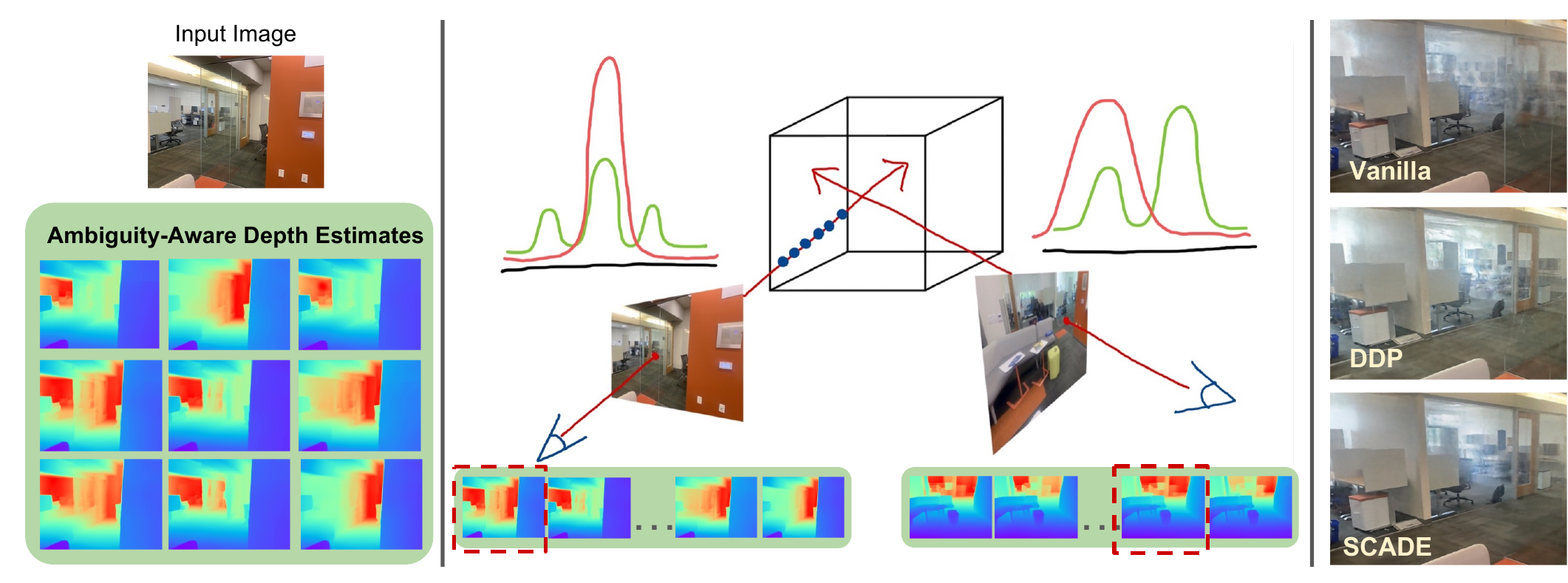}
    \vspace{-0.3cm}
    \captionof{figure}{\textbf{SCADE Overview}. We present SCADE, a novel technique for NeRF reconstruction under \emph{sparse, unconstrained views} for in-the-wild indoor scenes. We leverage on generalizable monocular depth priors and address to represent the inherent ambiguities of monocular depth by exploiting our \textbf{ambiguity-aware depth estimates} (left). Our approach accounts for \emph{multimodality} of both distributions using our novel \emph{space carving loss} that seeks to \emph{disambiguate} and find the common mode to \emph{fuse} the information between different views (middle). SCADE enables better photometric reconstruction especially in highly ambiguous scenes, \eg non-opaque surfaces (right). \vspace{-0.1cm}
    }
    \label{fig:teaser}
\end{center}%
}]
\maketitle

\begin{abstract}
\vspace{-0.5cm}
Neural radiance fields (NeRFs) have enabled high fidelity 3D reconstruction from multiple 2D input views. However, a well-known drawback of NeRFs is the less-than-ideal performance under a small number of views, due to insufficient constraints enforced by volumetric rendering. To address this issue, we introduce SCADE, a novel technique that improves NeRF reconstruction quality on sparse, unconstrained input views for in-the-wild indoor scenes. To constrain NeRF reconstruction, we leverage geometric priors in the form of per-view depth estimates produced with state-of-the-art monocular depth estimation models, which can generalize across scenes. A key challenge is that monocular depth estimation is an ill-posed problem, with inherent ambiguities. To handle this issue, we propose a new method that learns to predict, for each view, a continuous, multimodal distribution of depth estimates using conditional \emph{Implicit Maximum Likelihood Estimation} (cIMLE). In order to disambiguate exploiting multiple views, we introduce an original space carving loss that guides the NeRF representation to fuse multiple hypothesized depth maps from each view and distill from them a common geometry that is consistent with all views. Experiments show that our approach enables higher fidelity novel view synthesis from sparse views. Our project page can be found at \href{https://scade-spacecarving-nerfs.github.io}{scade-spacecarving-nerfs.github.io}.

\end{abstract}
\section{Introduction}
\vspace{-0.2cm}
Neural radiance fields (NeRF)~\cite{mildenhall2020nerf} have enabled high fidelity novel view synthesis from dozens of input views. Such number of views are difficult to obtain in practice, however. When given only a small number of sparse views, vanilla NeRF tends to struggle with reconstructing shape accurately, due to inadequate constraints enforced by the volume rendering loss alone.

Shape priors can help remedy this problem. Various forms of shape priors have been proposed for NeRFs, such as object category-level priors~\cite{jang2021codenerf}, and handcrafted data-independent priors~\cite{Niemeyer2021Regnerf}. The former requires knowledge of category labels and is not applicable to scenes, and the latter is agnostic to the specifics of the scene and only encodes low-level regularity (e.g., local smoothness). A form of prior that is both scene-dependent and category-agnostic is per-view monocular depth estimates, which have been explored in prior work~\cite{deng2022depth,roessle2022dense}. Unfortunately, monocular depth estimates are often inaccurate, due to estimation errors and inherent ambiguities, such as albedo vs. shading (cf. check shadow illusion), concavity vs. convexity (cf. hollow face illusion), distance vs. scale (cf. miniature cinematography), etc. As a result, the incorrect or inconsistent priors imposed by such depth estimates may mislead the NeRF into reconstructing incorrect shape and produce artifacts in the generated views. 

In this paper, we propose a method that embraces the uncertainty and ambiguities present in monocular depth estimates, by modelling a probability distribution over depth estimates. The ambiguities are retained at the stage of monocular depth estimation, and are only resolved once information from multiple views are fused together. We do so with a principled loss defined on probability distributions over depth estimates for different views. This loss selects the subset of modes of probability distributions that are consistent across all views and matches them with the modes of the depth distribution along different rays as modelled by the NeRF. It turns out that this operation can be interpreted as a probabilistic analogue of classical depth-based space carving~\cite{space_carving}. For this reason, we dub our method \emph{Space Carving with Ambiguity-aware Depth Estimates}, or SCADE for short.

Compared to prior approaches of depth supervision~\cite{deng2022depth,roessle2022dense} that only supervise the moments of NeRF's depth distribution (rather than the modes), our key innovation is that we supervise the modes of NeRF's depth distribution. The supervisory signal provided by the former is weaker than the latter, because the former only constrains the value of an integral aggregated along the ray, whereas the latter constrains the values at different individual points along the ray. Hence, the supervisory signal provided by the former is 2D (because it integrates over a ray), whereas the supervisory signal provided by the our method is 3D (because it is point-wise) and thus can be more fine-grained.

An important technical challenge is modelling probability distributions over depth estimates. Classical approaches use simple distributions with closed-form probability densities such as Gaussian or Laplace distributions. Unfortunately these distributions are not very expressive, since they only have a single mode (known as ``unimodal'') and have a fixed shape for the tails. Since each interpretation of an ambiguous image should be a distinct mode, these simple unimodal distributions cannot capture the complex ambiguities in depth estimates. Na\"{i}ve extensions like a mixture of Gaussians are also not ideal because some images are more ambiguous than others, and so the number of modes needed may differ for the depth estimates of different images. Moreover, learning such a mixture requires backpropagating through E-M, which is nontrivial. Any attempt at modifying the probability density to make it capable of handling a variable number of modes can easily run into an intractable partition function~\cite{hinton2012practical,cho2013gaussian,ranzato2007unified}, which makes learning difficult because maximum likelihood requires the ability to evaluate the probability density, which is a function of the partition function. 

\begin{figure}[t]
    \centering
    \includegraphics[width=0.9\linewidth]{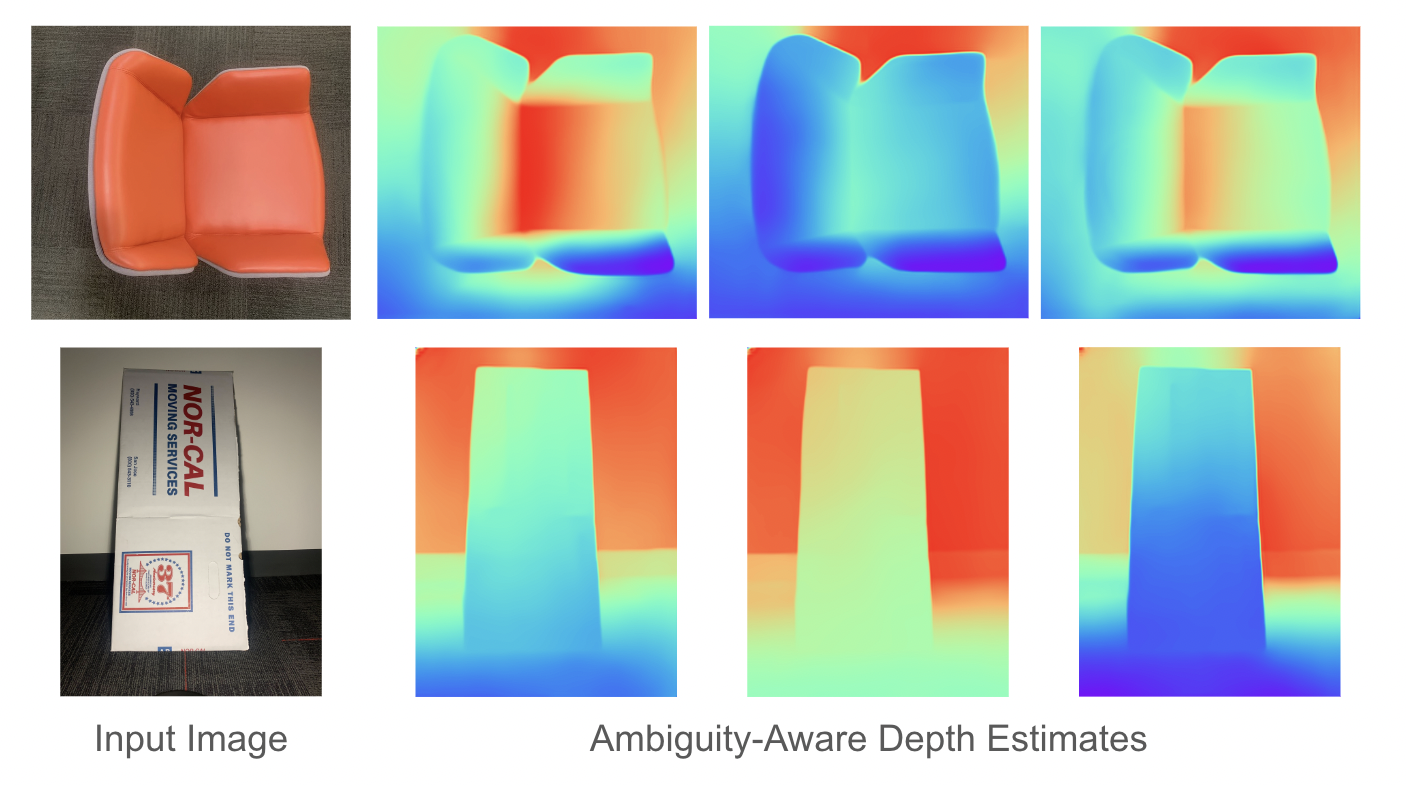}
    \vspace{-1.3em}
    \caption{{\textbf{Ambiguities from a single image}. We show samples from our ambiguity-aware depth estimates that is able to handle different types of ambiguities. (Top) An image of a chair taken from the top-view. Without context of the scene, it is unclear that it is an image of a chair. We show different samples from our ambiguity-aware depth estimates that are able to capture different degrees of convextiy. (Bottom) An image of a cardboard under bad lighting conditions that capture the albedo vs shading ambiguity that is also represented by our different samples.}}
    \label{fig:ambiguities}
    \vspace{-1.3em}
\end{figure}

To get around this conundrum, we propose representing the probability distribution with a set of samples generated from a neural network. Such a distribution can be learned with a conditional GAN; however, because GANs suffer from mode collapse, they cannot model multiple modes~\cite{isola2017image} and are therefore unsuited to modelling ambiguity. Instead, we propose leveraging conditional Implicit Maximum Likelihood Estimation (cIMLE)~\cite{Li2020MultimodalIS,chimle} to learn the distribution, which is designed to avoid mode collapse. 

We consider the challenging setting of leveraging out-of-domain depth priors to train NeRFs on real-world indoor scenes with sparse views. Under this setting, the depth priors we use are trained on a different dataset (e.g., Taskonomy) from the scenes our NeRFs are trained on (e.g., ScanNet). This setting is more challenging than usual due to domain gap between the dataset the prior is trained on and the scenes NeRF is asked to reconstruct. We demonstrate that our method outperforms vanilla NeRF and NeRFs with supervision from depth-based priors in novel view synthesis. 

In summary, our key contributions include:
\begin{itemize}
    \item An method that allows the creation of NeRFs in unconstrained indoor settings from only a modest number of 2D views by introducing a sophisticated way to exploit ambiguity-aware monocular depth estimation. 
    \item A novel way to sample distributions over image depth estimates based on conditional Implicit Maximum Likelihood Estimation that can represent depth ambiguities and capture a variable number of discrete depth modes.
    \item A new space-carving loss that can be used in the NeRF formulation to optimize for a mode-seeking 3D density that helps select consistent depth modes across the views and thus compensate for the under-constrained photometric information in the few view regime.
\end{itemize}

\vspace{-0.2cm}
\section{Related Work}
\vspace{-0.1cm}
\paragraph{Novel View Synthesis and 3D Reconstruction.}Reconstructing the structure of a scene from a few views is a long standing problem in computer vision with a long literature. Neural Radiance Fields or NeRF~\cite{mildenhall2020nerf} revolutionized the field by showing how to use weights of an MLP to represent a scene that is rendered using volume rendering~\cite{max1995optical}. A key innovation was the use of positional encoding~\cite{tancik2020fourier, vaswani2017attention} to increase the effective capacity of the MLPs that model a emitted radiance and density as a function of position and viewing direction. Extensions of NeRF include works on unconstrained photo collections~\cite{martinbrualla2020nerfw}, dynamic scenes~\cite{li2021neural}, deformable scenes~\cite{park2021nerfies}, and reflective materials~\cite{verbin2022refnerf, bi2020neural}.

A critical shortcoming of NeRF is its reliance on having many input views of a scene. Several approaches have been proposed, including adding patch likelihood losses~\cite{Niemeyer2021Regnerf}, data-driven priors~\cite{yu2021pixelnerf, tancik2020meta}, semantic consistency prior~\cite{Jain_2021_ICCV}, image features~\cite{wang2021ibrnet}, or surface~\cite{Niemeyer2021Regnerf, verbin2022refnerf}, occupancy~\cite{Oechsle2021ICCV}, and depth~\cite{wei2021nerfingmvs, kendall2017uncertainties, roessle2022dense} priors. DS-NeRF~\cite{deng2022depth} uses the sparse point reconstructions recovered during Structure-from-Motion, to supervise the depth of sparse points in the recovered NeRF. In the same spirit, DDP~\cite{roessle2022dense} uses a depth completion network, that takes as input sparse point cloud projected to one an input view, and produces a depth estimate. Both works~\cite{kendall2017uncertainties,roessle2022dense } model depth with uncertainty but only supervise with \emph{moments} of NeRF's depth distribution. 
In contrast, our work is able to represent \emph{multimodal} depth estimates, which can handle the inherit ambiguities of depth estimation, and is able to seek the modes for each depth estimate to make a consistent prediction of the scene structure.

Inspired by traditional multi-view stereo, MVS-NeRF~\cite{mvsnerf} and RC-MVSNet~\cite{chang2022rcmvsnet} incorporate the use cost volumes~\cite{gu2019cas} to NeRF. To construct the cost volumes, these works look for agreement between features to look for correspondences, which is difficult under the setting of large variations in appearance or viewpoint. In contrast, our approach introduces a novel space carving loss that does not rely on feature correspondences and instead directly defines the loss in 3D.

Another line of works focus on geometry reconstruction~\cite{wang2021neus, yariv2021volume, yu2022monosdf} by using the volumetric rendering scheme to learn a neural SDF representation. These works tackle a different problem on modeling geometry reconstruction, unlike NeRFs whose focus is modeling appearance for novel view synthesis, where our work falls under.
\vspace{-0.5cm}
\paragraph{Depth Estimation from Single View.}Depth estimation from a single image is a complex task due multiple ambiguities, including scene scale and shading / albedo ambiguities. Early efforts used MRFs to compute predictions on superpixels~\cite{saxena2008make3d}. The advent of consumer depth cameras like Kinect~\cite{zhang2012microsoft} enabled the acquisition of larger scale indoor 3D datasets~\cite{silberman2012indoor}, leading to methods that used deep neural networks to predict depth from a single color image~\cite{eigen2014depth}. Nonetheless, scaling datasets for depth estimation remained a challenge. Deep3D~\cite{xie2016deep3d} enabled the creation of stereo views by training on a large dataset of stereo movies, while~\cite{godard2017unsupervised} used the self-supervision of left-right consistency to learn depth from stereo views from driving cars. This approach was extended to videos where the egomotion is also estimated and self-supervision happens across time~\cite{zhou2017unsupervised}. MegaDepth~\cite{li2018megadepth} uses the depth from SfM reconstructions of internet photos, together with semantic segmentation that conveys ordinal depth supervision cues, i.e. transient objects must be in front of the static scenes. Most recent, multi-task learning has shown promise in training depth estimators that work well in out-of-domain data~\cite{ranftl2020towards}.

In some cases, surface normal estimation has fewer pitfalls than direct depth estimation, as it suffers somewhat less from scale ambiguities. GeoNet~\cite{qi2018geonet} jointly estimates depth and surface normals, then uses the geometric relation between them to refine the estimates. 
LeRes~\cite{yin2022towards} uses a second geometry reasoning module that refines a focal length estimate and the depth itself. Our depth estimation approach is derived from LeReS without the second stage, as the focal length is already estimated during SfM for NeRF scenes.

Most relevant is the work of Kendall and Gal~\cite{kendall2017uncertainties}, that learns depth estimation in a Bayesian setting, where their objective maximizes the likelihood under a predicted distribution. In our work, we use instead a multimodal depth estimation technique built on cIMLE~\cite{Li2020MultimodalIS}, which makes our prior robust to ambiguities in depth prediction.

\section{Background}
\subsection{Neural Radiance Fields (NeRF)}
\vspace{-0.2cm}
A neural radiance field~\cite{mildenhall2020nerf}, or NeRF for short, represents a field in 3D space, where each point represents an infinitesimal particle with a certain opacity that emits varying amounts of light along different viewing directions. The opacity at a point is represented as a volume density, and the amount of emitted light is represented as a colour. Mathematically, a NeRF is represented as two parameterized functions,  a volume density $\sigma_{\theta}: \mathbb{R}^3 \to \mathbb{R}_{\geq 0}$ and a colour $\mathbf{c}_{\psi}: \mathbb{R}^3 \times S^2 \to [0,255]^3$. The former maps a 3D coordinate $\mathbf{x} \in \mathbb{R}^3$ to a volume density $\sigma \in \mathbb{R}_{\geq 0}$, and the latter maps a 3D coordinate $\mathbf{x} \in \mathbb{R}^3$ and a viewing direction $\mathbf{d} \in S^2$ to a colour $\mathbf{c} \in [0,255]^3$. 

To render a NeRF from a view, we shoot rays from each pixel on the camera sensor and integrate over the product of colour, volume density and visibility along the ray to arrive at each pixel value. Visibility at a point is represented with transmittance, which accumulates the exponentiated negative volume density multiplicatively up to the point. The higher transmittance at a point, the more visible the point is from the camera. Mathematically, if the camera ray is $\mathbf{r}(t) = \mathbf{o}+t\mathbf{d}$, where $\mathbf{o}$ is the camera centre and $\mathbf{d}$ is the ray direction, the pixel value $\hat{I}_{\theta,\psi}(\mathbf{o}, \mathbf{d})$ we would have is:
\begin{equation}
\begin{split}
    &\hat{I}_{\theta,\psi}(\mathbf{o}, \mathbf{d}) := \int_{t_n}^{t_f} T_{\theta,\mathbf{o},\mathbf{d}}(t)\sigma_\theta(\mathbf{o}+t\mathbf{d})\mathbf{c}_\psi(\mathbf{o}+t\mathbf{d}, \mathbf{d}) \: \mathrm{d}t\\
    &\text{where } \; T_{\theta,\mathbf{o},\mathbf{d}}(t) = \exp(-\int_{t_n}^t\sigma_\theta(\mathbf{o}+s\mathbf{d}) \: \mathrm{d}s). 
\end{split}
\end{equation}
In the expressions above, $t_n$ and $t_f$ denote the points the ray intersects with the near plane and far plane respectively. 

\subsection{Inverse Rendering with NeRF}
Given a set of real-world images from known views, inverse rendering aims to find the scene whose rendered images match the real-world images. If we use $I(\mathbf{o}, \mathbf{d})$ to denote the pixel value for the ray $\mathbf{r}(t) = \mathbf{o}+t\mathbf{d}$, this problem can be cast as an optimization problem, namely:
\begin{equation}
\begin{split}
\min_{\theta,\psi} \sum_\mathbf{o} \sum_\mathbf{d} \Vert \hat{I}_{\theta,\psi}(\mathbf{o}, \mathbf{d}) - I(\mathbf{o}, \mathbf{d}) \Vert_2^2
\end{split}
\end{equation}

\vspace{-0.2cm}
If $\hat{I}_{\theta,\psi}(\mathbf{o}, \mathbf{d})$ is differentiable w.r.t. the parameters $\theta,\psi$, this problem can be tackled straightforwardly with gradient-based optimization. To enable this, the volume density function $\sigma_{\theta}$ and the colour function $\mathbf{c}_{\psi}$ are chosen to be neural networks. As a result, inverse rendering amounts to training the NeRF to reconstruct the scene in a way that is consistent with the images that are given. Once the NeRF is trained, novel view synthesis can be achieved by rendering the NeRF from new, unseen views. 

Since this optimization problem is underconstrained, in general many views are needed to reconstruct the scene accurately. Yet, in practical applications, typically only a few views are available, hence priors are needed to provide sufficient constraints. In this paper, we consider priors in the form of monocular depth estimates from each view. The monocular depth estimates are trained on different datasets from the dataset NeRF is trained on to simulate real-world conditions where there is a domain gap between what the prior was trained on and what the NeRF is trained on. 
\vspace{-0.2cm}
\subsection{Ray Termination Distance}
\vspace{-0.2cm}
In order to leverage depth priors, a natural way is to use them to constrain the ray termination distance. Because NeRF can represent non-opaque surfaces, there is no single ray termination distance. Instead, there is a distribution of ray termination distances. The cumulative distribution function (CDF) of this distribution represents the probability that the ray terminates before reaching a given point. It turns out that the probability density of this distribution $f_{\theta,\mathbf{o},\mathbf{d}}(t)$ is given by~\cite{deng2022depth}:
\begin{equation}
\begin{split}
f_{\theta,\mathbf{o},\mathbf{d}}(t) = T_{\theta,\mathbf{o},\mathbf{d}}(t) \sigma_\theta(\mathbf{o}+t\mathbf{d})
\label{eq:nerf_pdf}
\end{split}
\end{equation}

\vspace{-0.4cm}
\section{Method}
\subsection{Supervising Ray Termination Distance}
\label{section:sup_ray_term}
\vspace{-0.1cm}
If all surfaces were opaque and the ground truth depth were given, supervising NeRF with the ground truth depth would be straightforward. All that is needed is to make the NeRFs distribution of ray termination distances as close as possible to a delta distribution centred at the ground truth depth. However, when the ground truth depth is not known, we have to use monocular depth estimates as supervision. The challenge is that unlike ground truth depth, monocular depth estimates may not be stereo-consistent, due to estimation error and inherent ambiguity. As a result, it may not be possible to find a scene that is consistent with the depth estimates from every view. Therefore, in order to reconstruct the scene, we must model the uncertainty and ambiguity in the depth estimates, which can be mathematically characterized by distributions. Such ambiguities cannot be resolved from a single view and can often be resolved from multiple views. So we need an method that can fuse together uncertain depth estimates from different views. 

A natural way of representing uncertainty is through a probability distribution. When there are ambiguities, the distribution of depth estimates are typically \emph{multimodal}, i.e., their probability densities have disjoint peaks, where each mode represents one interpretation of the scene structure. For example, albedo vs shading ambiguity as shown in Figure~\ref{fig:ambiguities} can result in multiple plausible interpretations and make the distribution multimodal. 

An added challenge arises when non-opaque surfaces are present. In this case, there is no single ground truth depth, and so the distribution of ground truth depths is multimodal. So even after training the NeRF, the distribution of ray termination distances induced by NeRF should be multimodal, as illustrated in the glass walls in Figure~\ref{fig:teaser}. 

Prior methods compute moments (i.e., mean and/or variance) of the distribution of depth estimates or ray termination distance induced by NeRF. For example, DS-NeRF~\cite{deng2022depth} predicts the mean depth estimate and uses reprojection error as a proxy for the variance. It then minimizes the KL divergence from the distribution of ray termination distances induced by NeRF to a Gaussian whose moments match those of the depth estimates. DDP~\cite{roessle2022dense} fits a Gaussian to distribution of ray termination distances and maximizes the likelihood of the depth estimate under the Gaussian.

Distilling complex distributions to moments has the drawback of ignoring their possible multimodality. Because both the distribution of depth estimates and ray termination distances are potentially multimodal, it is important to handle the multimodality.



\begin{figure}[t]
    \centering
    \includegraphics[width=\linewidth]{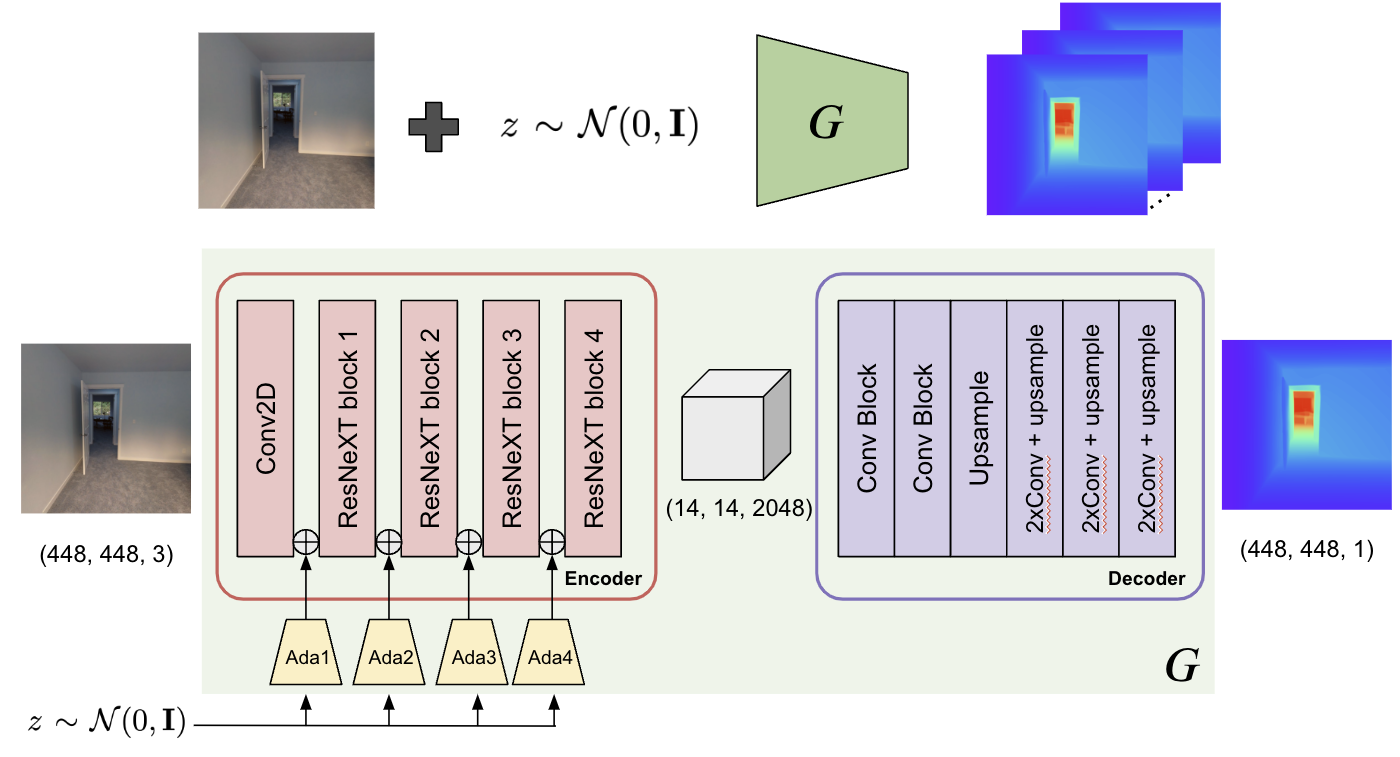}
    \vspace{-0.4cm}
    \caption{{Network Architecture for our Ambiguity-aware Depth Estimates.}}
    \label{fig:ambiguity-aware-prior}
    \vspace{-0.6cm}
\end{figure}
\vspace{-0.1cm}

\subsection{SCADE}
We now present our novel method \emph{Space Carving with Ambiguity-aware Depth Estimates} (\textbf{SCADE}). Our contributions addresses the existing problems and drawbacks mentioned in the previous section. First, SCADE accounts for the \emph{multimodality} of both the distributions of monocular depth estimates and ray termination distances that arise due to inherent ambiguities (Sec~\ref{sec:ambiguity-aware}) and non-opaque surfaces. Second, SCADE is also able to resolve these ambiguities by fusing together information from multiple views through our space carving loss (Sec~\ref{sec:space_carving}) that uses the reverse cross entropy loss. Because this a loss on the distribution rather than moments of the distribution (which was the paradigm in prior work~\cite{deng2022depth,roessle2022dense}), it achieves supervision in 3D rather than 2D. Third, our loss formulation is sample-based (Sec~\ref{sec:sampling}) and so is computationally efficient to optimize.
\vspace{-0.3cm}

\subsubsection{Our Ambiguity-aware Depth Estimates}
\label{sec:ambiguity-aware}
\vspace{-0.2cm}
We handle multimodality of monocular depth distributions by introducing our ambiguity-aware depth estimation module to account for ambiguities in depth estimation from a single view (Figure~\ref{fig:teaser}, \ref{fig:ambiguities}). In contrast to existing monocular depth estimation networks~\cite{Wei2021CVPR} that predict single point estimates for depth, we model the inherent uncertainty in monocular depth estimation by representing the distribution as a \emph{set of samples} generated by a neural network. We propose leveraging conditional Implicit Maximum Likelihood estimation (cIMLE)~\cite{Li2020MultimodalIS} to learn a multimodal distribution of depth estimates. We chose cIMLE rather than conditional GANs~\cite{karras2019style} in order to avoid mode collapse, which can lead to a unimodal distribution. 




Figure~\ref{fig:ambiguity-aware-prior} shows our network architecture. Concretely, we combine cIMLE with a state-of-the-art monocular depth estimation network, LeReS~\cite{Wei2021CVPR}. Our ambiguity-aware depth estimation module ($G$) takes an input image ($I$) and a latent code $z \sim \mathcal{N}(0, \mathbf{I})$ and outputs a conditional depth sample $G(I, z)$ for the input image $I$. To inject randomness into the network, we follow the technique used in~\cite{karras2019style} and incorporate AdaIn layers into the network backbone that predicts a scale and shift to the intermediate network features. Specifically, we added four AdaIn layers to the encoder of the depth estimation backbone.

\vspace{-0.3cm}

\subsubsection{Sample-based Losses on Distributions}
\label{sec:sampling}
\vspace{-0.2cm}

We desire to achieve an agreement between the distributions from our ambiguity-aware prior and the ray termination distance from NeRF to constrain the NeRF optimization under the sparse view set-up using our learned prior. As distributions are continuous, taking their analytical integral is computationally intractable. Thus, in order to define a differentiable loss function for our optimization problem, we result to a sample-based loss on these two distributions.
\vspace{-0.3cm}
\paragraph{Samples from depth prior.} As previously described, we leverage cIMLE to sample from our ambiguity-aware depth prior. Sampling  $z_1, ..., z_M \sim \mathcal{N}(0, \mathbf{I})$, we get depth map samples $G(I, z_1), ..., G(I, z_M)$ for input image $I$. Hence we denote corresponding samples to \emph{estimate} and represent the distribution for ray $\mathbf{r}$ as $y_1, y_2, ..., y_M \sim G_{\mathbf{o},\mathbf{d}}$.
\vspace{-0.3cm}
\paragraph{Samples from NeRF.} For the distribution given by the ray termination distances modelled by NeRF, we sample from its probability density function $f_{\theta, \mathbf{o}, \mathbf{d}}$ (Eq.~\ref{eq:nerf_pdf}) with inverse transform sampling. Concretely, we define the probability mass function for ray $\mathbf{r}$ as $f_{\theta, \mathbf{o}, \mathbf{d}}(t_i)$as given in Eq.~\ref{eq:nerf_pdf} with samples $t_1, ..., t_K$, where $t_i \in [t_n, t_f] \forall i$. By inverse transform sampling, we first compute the cumulative distribution function (CDF) of the ray termination distribution at samples $t_i$, which we denote as $F_{\theta, \mathbf{o}, \mathbf{d}}(t_i) = \sum_{j<i}f_{\theta, \mathbf{o}, \mathbf{d}}(t_i)$. Thus our samples $x_1, x_2, ..., x_N$ of ray termination distances from NeRF are then given by:
\begin{equation}
    \begin{split}
       x_i = F^{-1}_{\theta, \mathbf{o}, \mathbf{d}}(u_i), \\
       \text{where } u_i \sim \mathcal{U}(0, 1)
    \end{split}
\end{equation}


\begin{figure}[t]
    \centering
    \includegraphics[width=0.8\linewidth]{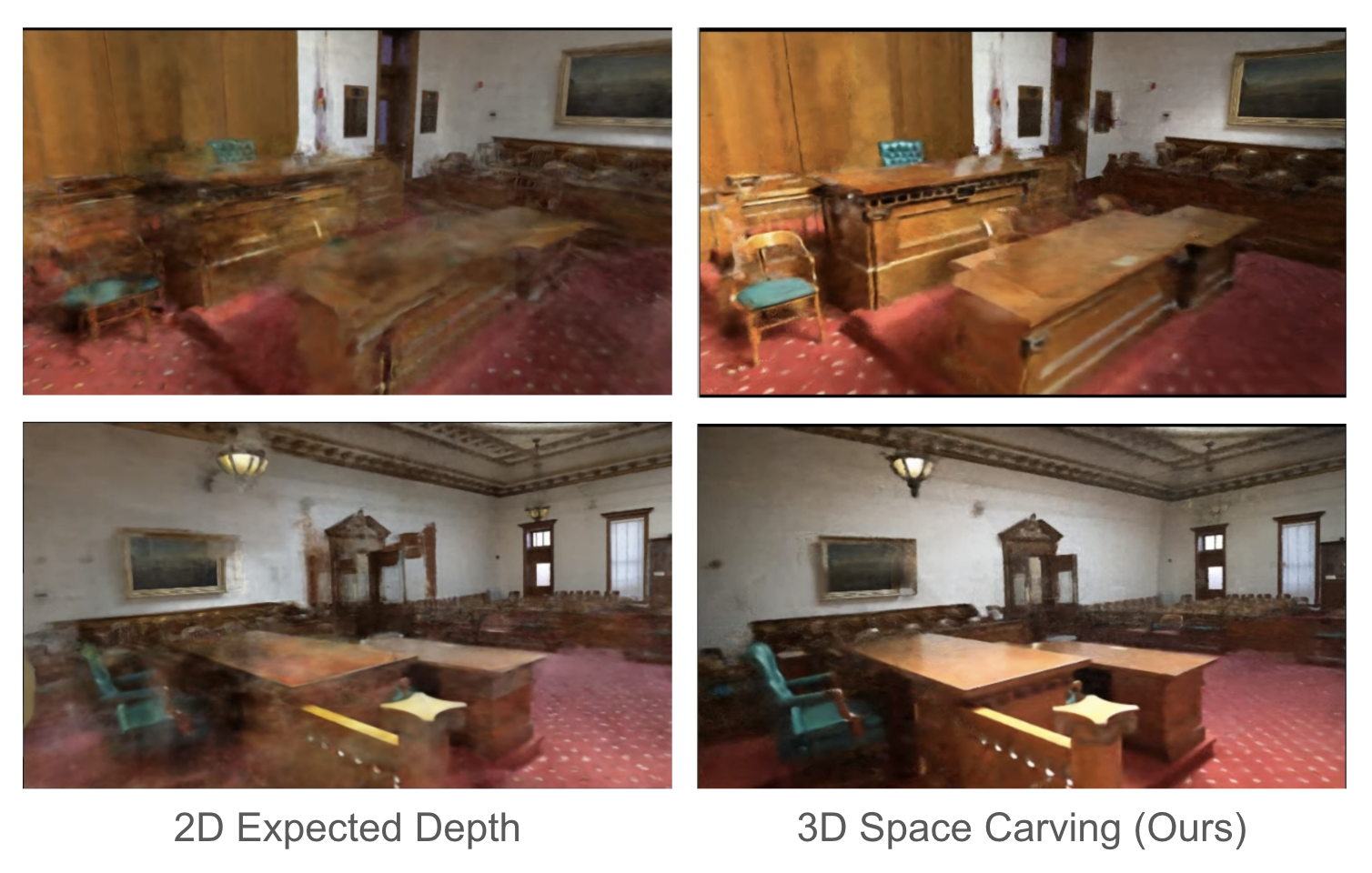}
    \vspace{-0.4cm}
    \caption{Our space carving loss supervises the ray termination in 3D as opposed to existing approaches that supervise on 2D expected depth. As shown, this clears out clouds of dust in space. Moreover supervising in 3D allows for wide baseline input views and still finds the common mode that correctly snaps the surfaces.}
    \label{fig:spacecarving}
    \vspace{-0.5cm}
\end{figure}
\vspace{-0.4cm}
\subsubsection{Our Space Carving Loss}
\label{sec:space_carving}
\vspace{-0.2cm}
We now introduce our novel space carving loss utilizes the samples from both multimodal distributions to constrain the NeRF optimization. The goal is to find a subset of modes captured in the monocular depth distributions from each view that are globally consistent. 
In doing so, we can fuse together the information from the different views, since the inherent ambiguities are only resolved given information from multiple views. This would allow us to find a common shape that is consistent across all views and ``snap'' objects surfaces together. 

We thus desire a loss that has \textbf{mode seeking} behavior, that is the modes of the distribution being trained should be a subset of the modes of the distribution that provides supervision. One such loss is the reverse cross entropy, which is the cross entropy from the NeRF ray termination distance distribution to the distribution of our ambiguity-aware depth estimates:
\begin{equation}
    H(f_{\theta, \mathbf{o}, \mathbf{d}}, G_{\mathbf{o},\mathbf{d}}) = -\mathbb{E}_{f_{\theta, \mathbf{o}, \mathbf{d}}}[\log G_{\mathbf{o},\mathbf{d}}].
    \label{ce_objective}
\end{equation}

As shown in the IMLE papers~\cite{Li2020MultimodalIS,chimle}, this is equivalent to penalizing the minimum of the L2 norm between the samples from the two distributions, please see~\cite{li2018implicit} for the full proof. Hence, our space carving loss is given by
\begin{equation}
    \mathcal{L}_\text{space\_carving}(\mathbf{r}) = \sum_{i\in[N]} \min_{j\in[M]} ||x_i - y_j||^2_2.
\end{equation}

Because our novel space carving loss operates on distributions rather than moments of distributions, we can have different supervision targets for different points along the ray. In contrast, if the loss were to only to operate on moments, 
which are integrals along the ray, this only yields the same supervision target for all points along the ray. Therefore, our loss allows us to supervise in 3D, unlike prior methods~\cite{deng2022depth, roessle2022dense} (which use losses on moments) that supervise in 2D. 3D supervision allows us to clear out dust in space, since in order to move a sample from the ray termination distance distribution farther in depth, the loss would have to decrease the probability density of \emph{all} the points in front of it. These advantages are highlighted in Figure~\ref{fig:spacecarving}.

Our total loss to optimize and train our NeRF model is given by
\begin{equation}
    \mathcal{L} = \mathcal{L}_{\text{photometric}} + \lambda \mathcal{L}_{\text{space\_carving}},
\end{equation}

\noindent where $\mathcal{L}_\text{photometric}$ is the standard MSE loss on the predicted and ground truth image. See Fig.~\ref{fig:teaser} for our overall pipeline.

\section{Results}
\vspace{-0.2cm}
In this section, we present our experimental evaluation to demonstrate the advantages of \textbf{SCADE}. 

\subsection{Datasets and Evaluation Metrics}
\vspace{-0.2cm}
We evaluate our method on ScanNet~\cite{dai2017scannet} and an in-the-wild dataset we collected~\footnote{Please also see supplement for results on the Tanks and Temples dataset~\cite{Knapitsch2017}.}.
We use the sparse-view ScanNet data used by DDP~\cite{roessle2022dense} in their evaluations, which comprises of three sample scenes each with 18 to 20 train images and 8 test images. To further test the robustness of our method, we further evaluated our method on three in-the-wild scenes collected using an iPhoneX. We captured sparse views in three different scenes -- basement, kitchen and lounge, where each scene has 18 to 23 train images and 8 test images. Following similar data preprocessing steps as DDP~\cite{roessle2022dense}, we ran SfM~\cite{schoenberger2016sfm} on all images to obtain camera poses for NeRF training. For quantitative comparison, we follow the original NeRF~\cite{mildenhall2020nerf} paper and report the PSNR, SSIM~\cite{wang2004image} and LPIPS~\cite{zhang2018unreasonable} on the novel test views. 

\subsection{Implementation Details}
\vspace{-0.2cm}
We train our ambiguity-aware prior on the Taskonomy~\cite{zamir2018taskonomy} dataset. We initialize the weights with the pretrained LeReS~\cite{Wei2021CVPR} model. We use $M=20$ depth estimates in our experiments. Please see supplement for additional implementation details. 

\begin{figure}[t]
\vspace{-2mm}
	\begin{center}
		\includegraphics[width=\linewidth]{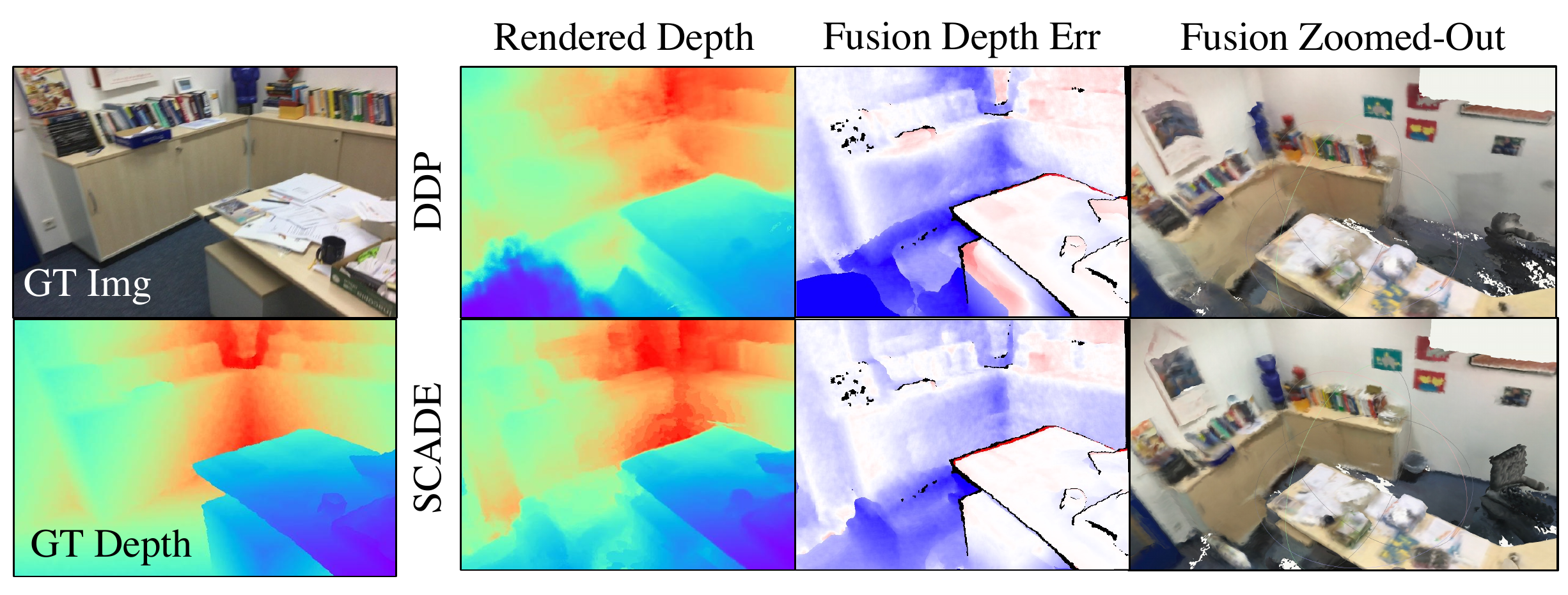}
	\end{center}
    \vspace{-0.6cm}
    \caption{\textbf{Depth and Fusion Comparison.} Depth map and fusion comparison between SCADE and DDP.}
	\vspace{-0.3cm}
	\label{fig:depth_fusion}
\end{figure}

\begin{table}
\centering
\setlength\tabcolsep{2pt}
\begin{tabularx}{\linewidth}{c|C|C|C}
    \toprule
    & PSNR $\uparrow$& SSIM $\uparrow$& LPIPS $\downarrow$\\
    \midrule
    Vanilla NeRF~\cite{mildenhall2020nerf} & 19.03 & 0.670 & 0.398\\
    NerfingMVS~\cite{Wei2021CVPR} & 16.29 & 0.626 & 0.502\\
    IBRNet~\cite{wang2021ibrnet}& 13.25 & 0.529 & 0.673\\
    MVSNeRF~\cite{mvsnerf} & 15.67 & 0.533 & 0.635\\
    DS-NeRF~\cite{deng2022depth} & 20.85 & 0.713 & 0.344\\
    DDP~\cite{roessle2022dense} & 19.29 & 0.695 & 0.368\\
    \midrule
    \textbf{SCADE} (Ours) & \textbf{21.54} & \textbf{0.732} & \textbf{0.292}\\
    \bottomrule
\end{tabularx}
\vspace{-0.2cm}
\caption{\textbf{ScanNet Results}. Results for DS-NeRF and NerfingMVS follow what was reported in prior literature~\cite{roessle2022dense}. Because our setting requires out-of-domain priors, the results for DDP are with out-of-domain priors. The results of DDP with in-domain priors are (20.96, 0.737, 0.236) for PSNR, SSIM and LPIPS, respectively.}
\label{tbl:ScanNet_results}
\vspace{-1.2\baselineskip} 
\end{table}

\begin{figure*}[t]
    \centering
    \includegraphics[width=0.95\linewidth]{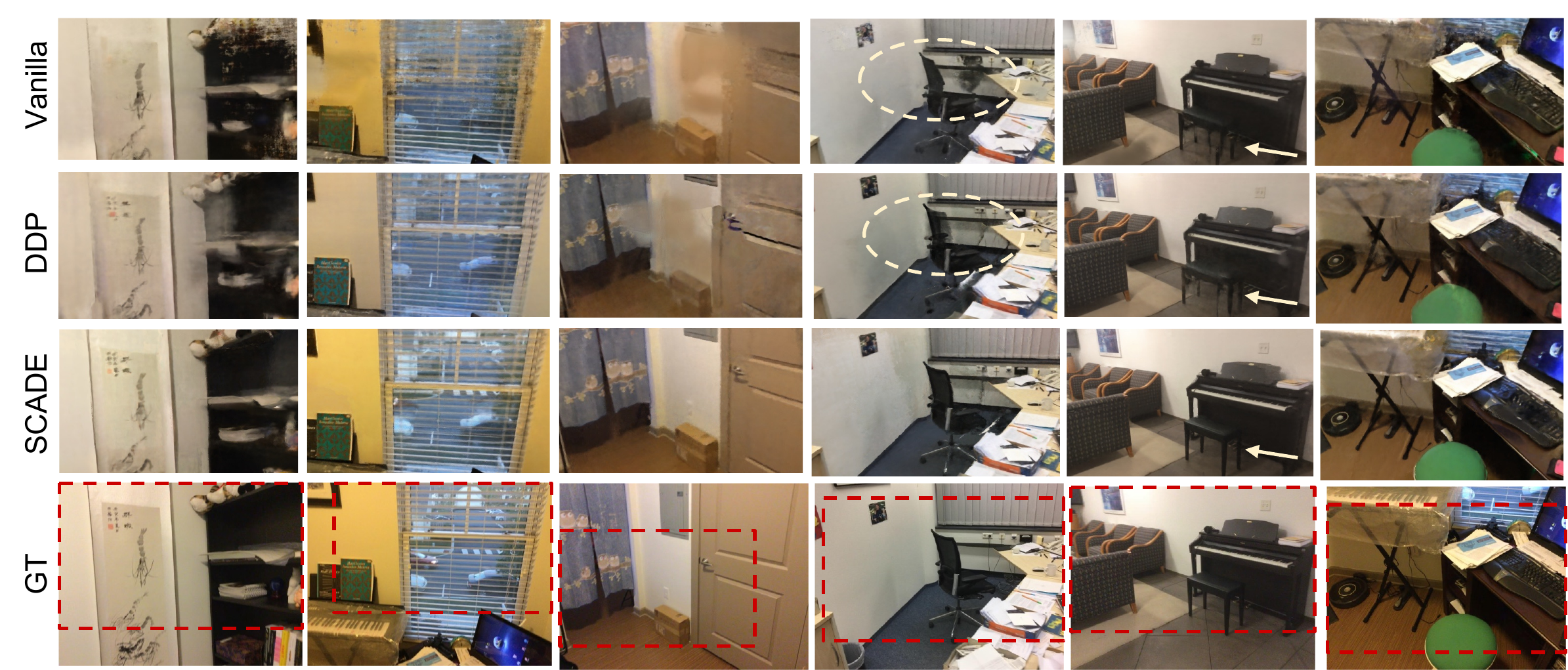}
    \vspace{-0.2cm}
    \caption{{\bf{Qualitative results on ScanNet}.}}
    \vspace{-0.4cm}
    \label{fig:ScanNet}
\end{figure*}

\subsection{Experiments on ScanNet}
\vspace{-0.1cm}
We evaluate on ScanNet following the evaluation from DDP~\cite{roessle2022dense}. We compare to the original NeRF~\cite{mildenhall2020nerf} (\textbf{Vanilla NeRF}) and the recent state-of-the-art NeRF with depth prior-based supervision, Dense Depth Prior~\cite{roessle2022dense} (\textbf{DDP}). Table~\ref{tbl:ScanNet_results} shows SCADE quantitatively outperforming the baselines. Because we are interested in simulating real-world conditions with a domain gap between the prior and the NeRF, our setting requires the use of out-of-domain priors. Since DDP uses a in-domain prior, we retrain DDP's depth completion network using their official code and hyperparameters on Taskonomy~\cite{zamir2018taskonomy}, \ie the same out-of-domain dataset that our prior is trained on.

Figure~\ref{fig:ScanNet} shows our qualitative results. As shown, compared to the baselines, SCADE is able to avoid producing the clouds of dust that are present in the results of baselines (first, third and last column). Moreover, SCADE is also able to snap and recover objects in the scene such as the details on the blinds (second column), the back of the chair (fourth column) and the legs of the piano stool (fifth column). Moreover, Fig.~\ref{fig:depth_fusion} shows rendered depthmaps and fusion results using~\cite{zeng20163dmatch} on a Scannet scene. Notice that SCADE is able to recover better geometry compared to DDP -- see corner of the calendar, cabinets and office chair in the right image.

\subsection{Experiments on In-the-Wild Data}
\vspace{-0.1cm}
To further test the robustness of SCADE, we further evaluate on data collected in-the-wild with a standard phone camera (iPhoneX). The intrinsics for different views are different and are not known to the algorithm. As shown in Table~\ref{tbl:wild_results}, \textbf{SCADE} with the same out-of-domain prior (trained on Taskonomy) outperforms the baselines on in-the-wild captured data, which demonstrates its robustness. Figure~\ref{fig:in_the_wild} shows qualitative examples. Interestingly, we are able to recover the room behind the glass better the baselines, whose results are more blurry (first, third and fourth column). We are also able to recover objects behind the glass such as the monitor and chair (first column). As discussed in Sect. \ref{section:sup_ray_term}, non-opaque surfaces are the most challenging because both the ground truth distribution of ray termination distances induced by NeRF and the distribution of estimated depths from the prior must be multimodal. The fact that we do well validates the ability of our model to capture the multimodal distributions. 

Similar to the ScanNet results, we are able to recover crisp shapes of objects such as the thin table leg and the white chair (second and last column), and we are also able to clear up dust near the microwave and printer compared to the baselines (fifth and last column). 

\begin{figure*}[t]
    \centering
    \includegraphics[width=0.95\linewidth]{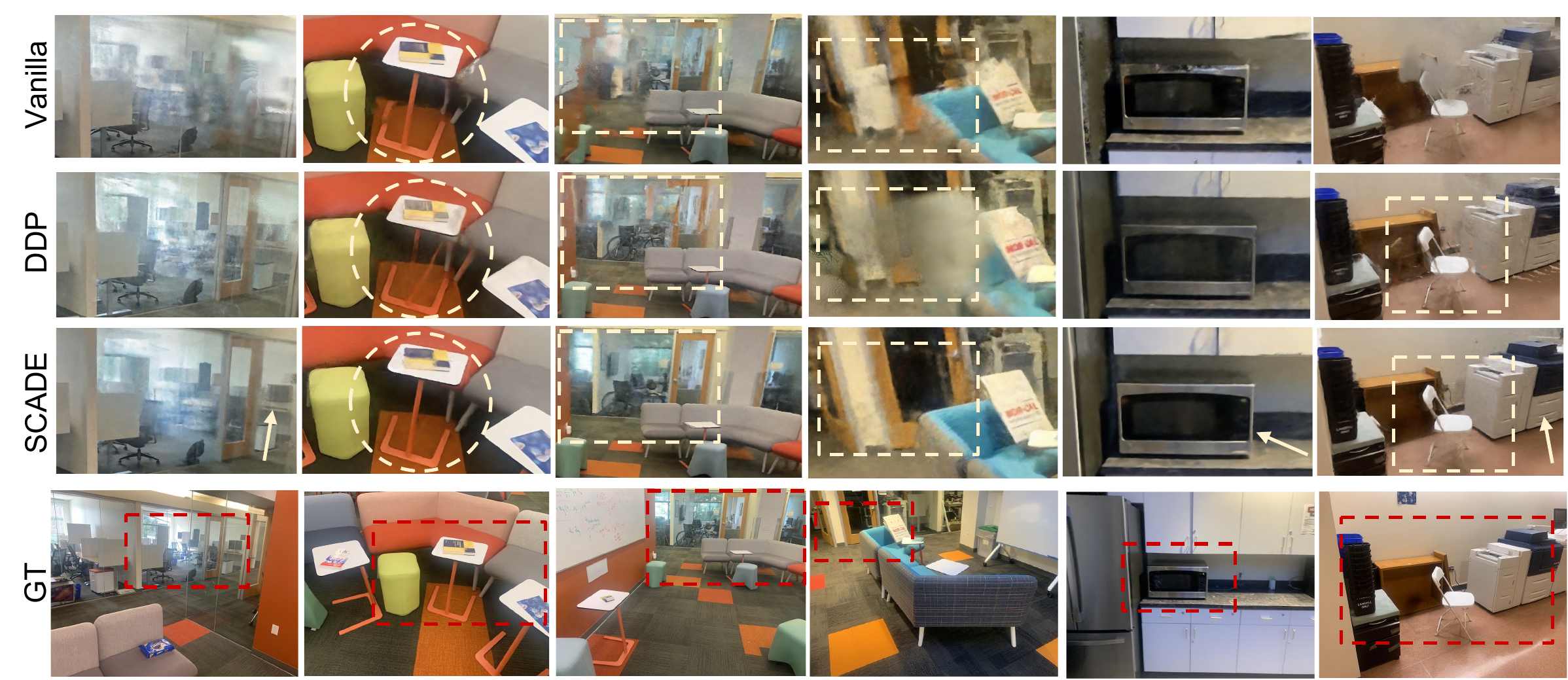}
    \vspace{-0.3cm}
    \caption{{\bf{Qualitative results on In-the-Wild data}.}}
    \label{fig:in_the_wild}
    \vspace{-0.4cm}
\end{figure*}

\begin{table}
\centering
\setlength\tabcolsep{2pt}
\begin{tabularx}{\linewidth}{c|C|C|C}
    \toprule
    &  PSNR $\uparrow$& SSIM $\uparrow$& LPIPS $\downarrow$\\
    \midrule
    Vanilla NeRF~\cite{mildenhall2020nerf} & 20.46 & 0.713 & 0.398 \\
    DDP~\cite{roessle2022dense} & 21.28 & 0.727 & 0.366 \\
    \midrule
    \textbf{SCADE} (Ours) & \textbf{22.82} & \textbf{0.743} & \textbf{0.347}\\
    \bottomrule
\end{tabularx}
\vspace{-0.2cm}
\caption{\textbf{In-the-wild Results}. }
\label{tbl:wild_results}
\vspace{-1.5\baselineskip} 
\end{table}

\subsection{Ablation Study}
\paragraph{Supervision on full distribution vs moments} We first validate the importance of using our novel space carving loss that supervises the full distribution of ray termination distance rather than just its moments. The latter integrates along the ray, and so provides one target value for the entire ray, which effectively makes it 2D supervision. The former provides different target values for different samples along the same ray, which makes it 3D supervision. We adapt our method to use 2D supervision proposed in the recent MonoSDF~\cite{yu2022monosdf}, which computes the expected ray termination distance and aligns the output depth map with a monocular depth estimation prior using the MiDaS~\cite{ranftl2020towards} loss. We use our learned prior as the monocular depth estimate. Table~\ref{tbl:ablation} validates the effectiveness of our space carving loss.

\vspace{-0.25cm}
\paragraph{Multimodal prior vs unimodal prior.} Our prior models does not constrain the depth distribution to a particular form and allows it to be multimodal. In contrast, DDP's depth completion prior assumes a Gaussian, which is unimodal. We ablate our prior against DDP's by sampling from their Gaussian distribution. We train with our space carving loss using both a single sample, as well as $M$ samples drawn from DDP's prior. As shown in Table~\ref{tbl:ablation}, while both perform better than the 2D supervision provided by MonoSDF, but they are subpar compared to using our multimodal prior.

\vspace{-0.25cm}
\paragraph{Multiple vs single sample.} Finally, we also ablate on using a single sample vs multiple samples for our space carving loss. For the single sample set-up, we supervise using the sample mean of the depth estimates, which is equivalent to the maximum likelihood estimate of the depth under Gaussian likelihood. Results in Table~\ref{tbl:ablation} show the importance of using multiple samples for our space carving loss.

\begin{table}
\centering
\setlength\tabcolsep{2pt}
\begin{tabularx}{\linewidth}{c|C|C|C}
    \toprule
    &  PSNR $\uparrow$& SSIM $\uparrow$& LPIPS $\downarrow$\\
    \midrule
     MonoSDF supervision & 20.13 & 0.710 & 0.332 \\
     DDP prior - single sample & 20.85 & 0.712 & 0.320 \\
     DDP prior - multiple samples & 21.00 & 0.718 & 0.316 \\
     Our prior - single sample & 21.22 & 0.714 & 0.318 \\
    \midrule
    \textbf{SCADE} (Ours) & \textbf{21.54} & \textbf{0.732} & \textbf{0.292}\\
    \bottomrule
\end{tabularx}
\vspace{-0.2cm}
\caption{\textbf{Ablation Study}. 
Results on the Scannet dataset. MonoSDF supervision refers to supervising with MonoSDF loss on expected ray termination distance using our prior. }
\vspace{-0.4cm}
\label{tbl:ablation}
\end{table}

\vspace{-0.2cm}
\paragraph{Sparsity.} We further show addt'l results under varying number of views in Tab~\ref{tbl:sparsity_table}. Note that in general sparsity is not fully reflected by the absolute number of views, because all else being equal, a larger scene requires more views to attain complete coverage. 
Following prior work~\cite{roessle2022dense}, we also report the average number of views that see the same point as another measure of sparsity. As shown, our setting is at the edge of the theoretical lower limit of 2 for general-purpose reconstruction, which shows view sparsity.

\begin{table}
{
\footnotesize
\centering
\setlength\tabcolsep{2pt}
\begin{tabularx}{\linewidth}{c|c|c|c|c}
    \toprule
    ave. \#views visible & 1.87 & 1.98 & 2.01 & 2.2 \\
    absolute \#views & 18 & 20 & 22 & 24 \\
    \midrule
    DDP / SCADE & 19.35/ \textbf{21.66} & 22.4/ \textbf{23.67}  &  23.10/ \textbf{24.00} & 23.56/ \textbf{24.84}  \\
    \bottomrule
\end{tabularx}
\vspace{-0.3cm}
\caption{PSNR for DDP~\cite{roessle2022dense}/SCADE on scene781 of Scannet.}
\label{tbl:sparsity_table}
}
\vspace{-0.1cm}
\end{table}

\begin{figure}[t]
\vspace{-2mm}
	\begin{center}
		\includegraphics[width=\linewidth]{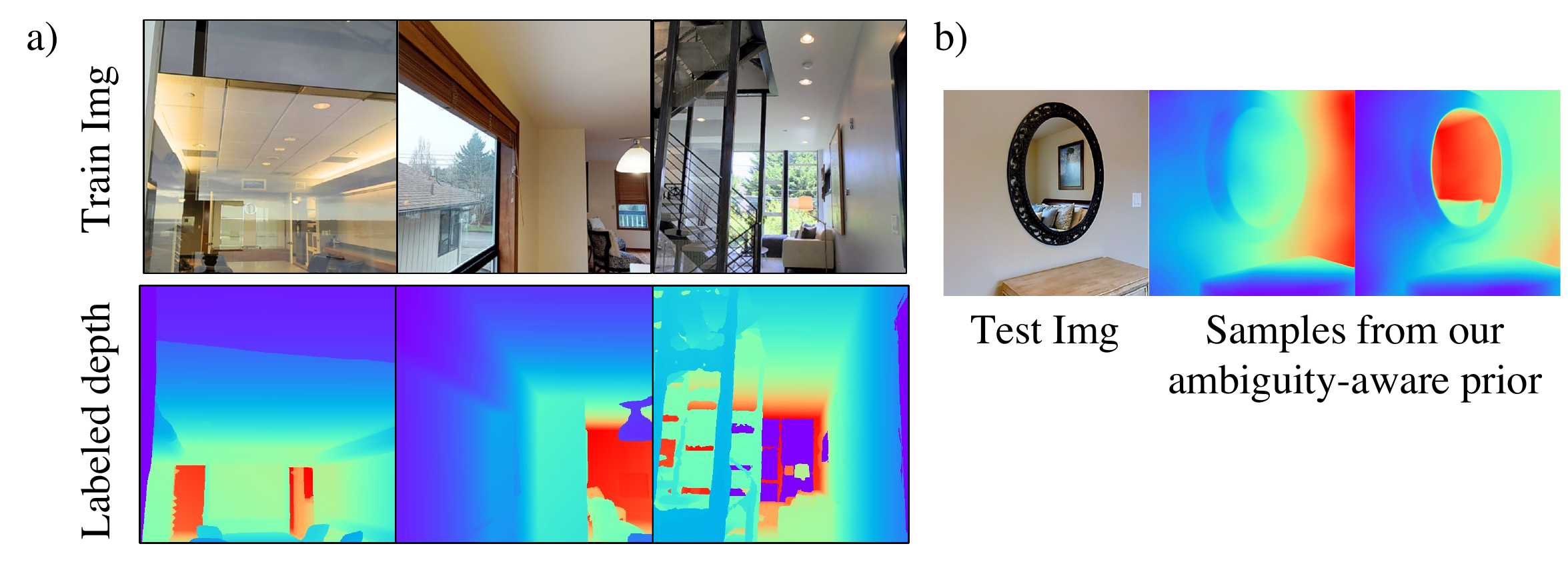}
	\end{center}
    \vspace{-0.6cm}
    \caption{\textbf{Depth Mode Discussion.} a) Train images from Taskonomy~\cite{zamir2018taskonomy} and their labels. Notice that non-opaque surfaces are labelled differently. b) Output of our prior on reflective surfaces.}
	\vspace{-0.4cm}
	\label{fig:ambiguity-discussion}
\end{figure}

\subsection{Discussion on our Ambiguity-Aware Prior}
\vspace{-0.2cm}
We provide addt'l details on our ambiguity-aware prior on i) how and why it works, ii) its performance on reflective surfaces, which is a common failure case for depth estimate, and iii) provide an intuition on how what our depth modes look like. i) We achieve variable depth modes by exploiting inconsistently labeled training data. As shown in Fig~\ref{fig:ambiguity-discussion}-a, in Taskonomy~\cite{zamir2018taskonomy}, different training images with non-opaque surfaces label depth differently: shooting through the glass (left), on the glass (middle), or a mixture of both (right). Despite multiple possible depth labels, each image only has one ground truth label. Training with cIMLE allows our prior to model these multiple possible (ambiguous) outputs through sampling\footnote{See supplement for more qualitative examples on our ambiguity-aware depth estimates.}, even when given only one label per image. Interestingly when testing our prior on a test image with a mirror, we find that it is able to capture variable modes on reflective surfaces, including the ``correct" flat surface as shown in Fig~\ref{fig:ambiguity-discussion}-b. An intuition of what the depth modes look like is for example if a ray intersects $n-1$ non-opaque surfaces, the $n$ modes are intersections with the non-opaque surfaces and the terminating point on the opaque surface.

\vspace{-0.2cm}
\section{Conclusion}
\vspace{-0.2cm}
In this paper, we present a new approach towards NeRF reconstruction that can work with a modest number of in-the-wild views of an indoor scene. We address the under-constrained nature of this problem by regularizing the NeRF optimization with additional depth estimates for each view. Our key technical contribution is to model multimodality in the depth estimates, which can capture inherent ambiguities in monocular depth estimation as well as the possible presence of non-opaque surfaces. We resolve ambiguities using a novel space carving loss that fuses the multimodal depth estimates from different views and seeks the modes that are consistent across views so as to arrive at a globally consistent 3D reconstruction. The improved recovery of shape and appearance enables higher fidelity novel view synthesis from sparse views. 
\vspace{-0.4cm}
\paragraph{Limitations and Future Work.} The performance of our method is constrained by the quality of the monocular depth priors. While we found that our prior generalizes well across domains, if the domain gap is too great, the performance of our method will degrade. A future direction would to be detect when this happens and dynamically adjust the strength of depth supervision in response. 
\vspace{-0.4cm}
\paragraph{Acknowledgements.} We sincerely thank the Teleportation and CCI team in Google for all the insightful discussions during the summer. We also thank Mirko Visontai for internship logistics, Guandao Yang for experiment set-up assist, and Weicheng Kuo for looking over the paper. 

{\small
\bibliographystyle{ieee_fullname}
\bibliography{egbib}
}

\renewcommand{\thesection}{A}
\renewcommand{\thetable}{A\arabic{table}}
\renewcommand{\thefigure}{A\arabic{figure}}

\clearpage
\section{Additional Results}
\subsection{Experiments on Tanks and Temples}
We conduct further experiments to test the robustness of SCADE. We evaluate on three scenes from the Tanks and Temples~\cite{Knapitsch2017} dataset, namely three large indoor rooms - Church, Courtroom and Auditorium scenes. The training set consists of 21, 26 and 21 sparse views for the Church, Courtroom and Auditorium scenes respectively, and the test set consists of 8 sparse views, so the amount of data is similar to that used in prior work~\cite{roessle2022dense}. We also followed similar data preprocessing steps as prior work~\cite{roessle2022dense} and ran SfM~\cite{schoenberger2016sfm} on all images to obtain camera poses for training.

As shown in Table~\ref{tbl:tnt_results}, \textbf{SCADE} trained with the same out-of-domain prior that we used for the other datasets (which was trained on Taskonomy~\cite{zamir2018taskonomy}) outperforms the baselines on the Tanks and Temples dataset as well. Moreover, Figure~\ref{fig:tanks_and_temples_quali} shows qualitative results. As shown, \textbf{SCADE} is able to recover objects better than the baselines such as the table in the Church, the group of chairs in the Courtroom (second column), and the rows of seats in the Auditorium (clearer in the side-view seats on the second column). Moreover, results also show that \textbf{SCADE} avoids clouds of dust such as the lights on the wall of the Church (second column), painting on the wall of the Courtroom (last column) and details on the repetitive seats of the auditorium.


\begin{table}
\centering
\setlength\tabcolsep{2pt}
\begin{tabularx}{\linewidth}{c|C|C|C}
    \toprule
    &  PSNR $\uparrow$& SSIM $\uparrow$& LPIPS $\downarrow$\\
    \midrule
    Vanilla NeRF~\cite{mildenhall2020nerf} & 17.19 & 0.559 & 0.457 \\
    DDP~\cite{roessle2022dense} & 19.18 & 0.651 & 0.361 \\
    \midrule
    SCADE & \textbf{20.13} & \textbf{0.662} & \textbf{0.358}\\
    \bottomrule
\end{tabularx}
\caption{\textbf{Quantitative results for the Tanks and Temples~\cite{Knapitsch2017} dataset}. }
\label{tbl:tnt_results}
\end{table}

\subsection{Video Demo}
Our project page \href{https://scade-spacecarving-nerfs.github.io}{scade-spacecarving-nerfs.github.io} shows a video trajectory from each of the three datasets in our experiments.
As shown on the Scannet scene, \textbf{SCADE} is able to better recover and crisp up the black chair; on the In-the-Wild data scene, the room and objects behind the glass wall are better captured, and finally, on the Tanks and Temples scene, the table on the center is more solid and also clear up dust on the wall and aisle of the church.

\begin{table}
\centering
\setlength\tabcolsep{2pt}
\begin{tabularx}{\linewidth}{c|C|C|C}
    \toprule
    &  PSNR $\uparrow$& SSIM $\uparrow$& LPIPS $\downarrow$\\
    \midrule
    $M=1\text{  }$ & 21.22 & 0.714 & 0.318 \\
    $M=5\text{  }$ & 21.05 & 0.722 & 0.304 \\
    $M=10\text{  }$ & 21.41 & 0.729 & 0.296\\
    $M=20\text{  }$ & 21.54 & \textbf{0.732} & \textbf{0.292}\\
    $M=40\text{  }$ & \textbf{21.61} & 0.729 & 0.293\\
    $M=80\text{  }$ & 21.58 & 0.729 & 0.295\\
    \bottomrule
\end{tabularx}
\caption{\textbf{Ablation on $M$}. We ablate on the number of depth estimates used from our ambiguity-aware prior to train SCADE.}
\label{tbl:m_ablation}
\end{table}

\subsection{Ablation on Number of Hypotheses $M$}
We further ablate on the number of estimates $M$ from our ambiguity-aware prior for training SCADE. Table~\ref{tbl:m_ablation} shows the quantative results for the Scannet dataset. As shown, the results in general improve as we increase the number of depth estimates as this gives us a better approximation of the depth distribution. We observe that the improvement is marginal as we increase the number of depth estimates beyond $20$. Hence, we use $M=20$ in our experiments.

\section{Implementation Details}
We train our ambiguity-aware prior with a batch size of 16 and use a learning rate of 0.001 for the base model, and 0.0001 for the MLP layers before AdaIn~\cite{huang2017arbitrary}. We use a latent code dimension of $32$, and we follow the standard cIMLE training strategy~\cite{Li2020MultimodalIS} and sample 20 latent codes per image and resample every 10 epochs. 

To train our NeRF model, we use a batch size of 1024 rays for 500k iterations. We use the Adam optimizer~\cite{kingma2014adam} with a learning rate of $5e^{-4}$ decaying to $5e^{-5}$ in the last 100k iterations. We use the same architecture as the original NeRF~\cite{mildenhall2020nerf}, which samples 64 points the coarse network and an additional 128 points for the fine network. Because the depth estimates are in relative scale, we directly optimize for the scale and shift for each input image. We directly optimize a 2-dim variable for \emph{each input image} that scales and shifts depth hypotheses. These variables are jointly optimized with the NeRF for the first 400k iters (learning rate of 1e-6), and are then kept frozen for the last 100k iters. 

\section{Ambiguity-Aware Depth Estimates}
\begin{figure}[t]
    \centering
    \includegraphics[width=\linewidth]{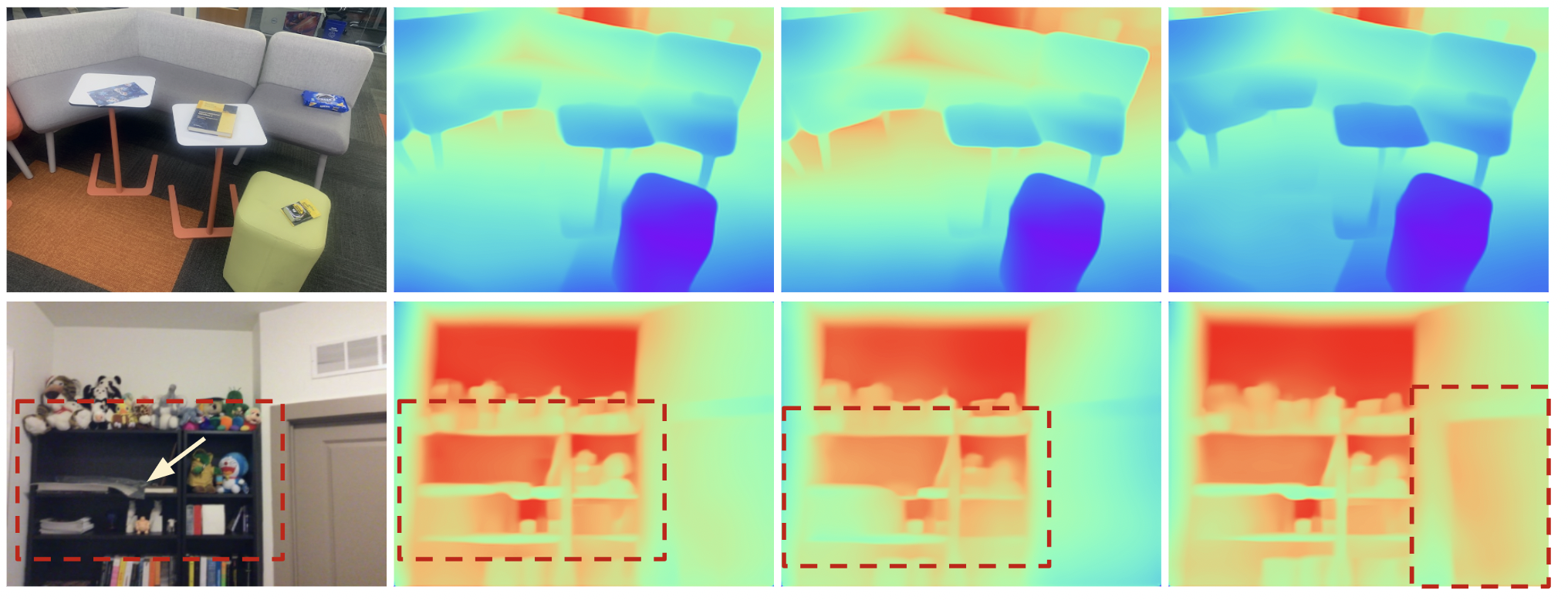}
    
    \caption{{\textbf{Depth Estimate Samples}. Here we show two examples of train images from scenes used in our experiments that show the ambiguity in (top) different degrees of convexity and (b) albedo vs shading ambiguity on the door frame and possible existence of an object inside the bookshelf. Please see Fig.~\ref{fig:depth_hypothesis_many} for more examples.}}
    \label{fig:depth_hyp}
    \vspace{-1.3em}
\end{figure}

We show some samples from the depth distribution of our multimodal prior from scenes in ScanNet and our in-the-wild data in Figures~\ref{fig:teaser_supp} and \ref{fig:depth_hyp}. We see that depth from a single input image is ambiguous as captured by our multimodal prior. In Figure~\ref{fig:teaser_supp}, we are able to capture the multimodality in ray termination distance caused by non-opaque glass surfaces. In Figure~\ref{fig:depth_hyp} (top row) we are able to capture different degrees of concavity of the sofa as well as the ambiguity in the depth of the far wall and floor. In Figure~\ref{fig:depth_hyp} (bottom row), we have ambiguity on the presence of a dark colored object on the boxed shelf and the depth of the door w.r.t. the door frame due to albedo vs shading ambiguity. Figure~\ref{fig:depth_hypothesis_many} shows more samples of our multimodal depth estimates on train images for the different scenes used in our experiments. Note that the depth map visualizations are normalized per image, i.e. the colors represent per image relative depth. 



\subsection{Adaptation of the cIMLE~\cite{Li2020MultimodalIS} proof for our Ambiguity-Aware Prior}
\theoremstyle{plain}
\newtheorem{thm}{\protect\theoremname}
\theoremstyle{plain}
\newtheorem{lem}[thm]{\protect\lemmaname}
\providecommand{\lemmaname}{Lemma}
\providecommand{\theoremname}{Theorem}

Here we show an adaptation of the proof provided in IMLE~\cite{li2018implicit} in the context of learning our ambiguity-aware depth estimates. Recall that we are given a set of input images $\{I_1, I_2, ..., I_n\}$ each with a corresponding ground truth depth map $D_1, D_2, ..., D_n$. As we know that monocular depth estimation is inherently ambiguous, we desire to learn a \emph{multimodal distribution of depth estimates} conditioned on an input image given only one ground truth lable (\ie depth map). 

Thus, we want to learn the network parameters $\phi$ for conditional distribution $G$ such that $G_\phi(I, z)$ models the distribution of depth estimates for a given input image $I\in \{I_1, I_2, ..., I_n\}$, where $z\sim \mathcal{N}(0, \mathbf{I})$ are latent codes sampled from a normal distribution.

Unlike GAN's~\cite{karras2019style} that optimize that each sample is similar to a ground truth data point, cIMLE~\cite{Li2020MultimodalIS} prevents mode collapse by instead enforcing that all ground truth data points are explained by at least one generated sample. Hence, in order to learn $\phi$, the objective function that we want to optimize is maximizing the sum of the likelihoods at the training examples. 

Consider our ambiguity-aware prior $G_{\phi, i}$, an implicit generative model, the likelihood induced by this model $P_{\phi, i}$ is computationally intractable to compute as it cannot be expressed in closed form. In this proof, we show that maximizing this likelihood is equivalent to optimizing a sample-based objective, making it tractable. We first i) rewrite the desired objective function (Sec~\ref{sec:objective_function}), ii) show its equivalence to the loss function used in training (Sec~\ref{sec:equivalence_loss_function}), then finally iii) show its equivalence to maximizing the sum of the likelihoods at the training examples, \ie the single ground truth depth maps associated with each image (Sec~\ref{sec:imle_convergence_proof}).

\subsubsection{Objective function}
\label{sec:objective_function}
Let's consider the following objective function:

\begin{equation}
\begin{split}
\max_{\phi}\mathcal{L}_{\{\delta_{i}\}_{i}}(\phi):=\max_{\phi}\mathbb{E}_{\{y_{i,j} \sim P_{\phi,i}\}_{i,j}}\Big[\frac{1}{n}\sum_{i=1}^{n}\frac{1}{w_{i}}\Big(\delta_{i}-\\
\frac{1}{M}\sum_{j=1}^{M}\Phi_{\delta_{i}}(d(y_{i, j},D_i)\Big)\Big]
\end{split}
\end{equation}

\noindent $y_{i, j}$ is a sampled depth estimate, \ie $y_{i, j}= G(I_i, z_j)$, $M$ is the number of samples drawn, $d(y_{i, j},D_i)$ denotes the distance between the sampled depth estimate and the given ground truth depth map for image $I_i$. 

$\delta_i >0$ denotes the threshold of the radius of the largest neighborhood that we are interested in, \ie the neighborhood around the ground truth data points (depth maps) where we are interested in having generated depth estimate samples at. This radius is dependent on the training example (hence the subscript $i$) as some examples may have a larger/smaller neighborhood of interest than others. $\Phi_{\delta_i}$ is a function we choose, which we will define below, and $w_i$ is a weighting factor that is also dependent on the training example.

Note that here, we reuse $\mathcal{L}$ to denote the likelihood, and it should not be confused with the notation for the loss functions in the main paper.\\

\noindent\textbf{Choosing $\Phi_\delta$.}\\

For $\delta>0$ (threshold on the radius), $\Phi_\delta$ is chosen as
\begin{equation}
\begin{split}
\Phi_{\delta}(t)=
\begin{cases}
t & 0 \leq t\leq\delta\\
\delta & t>\delta
\end{cases},
\end{split}
\end{equation}

\noindent Intuitively, this assigns the random variable t, which is our case will be the the distance $d(\cdot)$ between the ground truth depth and a sampled depth estimate, to a value depending on the radius threshold $\delta$. Any distance larger than $\delta$, \ie is the sampled estimate is far enough, is set to $\delta$.

Consequently, the chosen antiderivative is shown below
\[
\Phi_{\delta}'(t)=\begin{cases}
1 & 0 \leq t\leq\delta\\
0 & t>\delta
\end{cases}
\]

\noindent\textbf{Relating to model distribution $P_{\phi, i}$}\\
Three lemmas written below tie together the likelihood $\mathcal{L}_{\delta}$ to the objective function. 

\begin{lem}
Let $Y$ be a non-negative random variable and $f$ be a continuous
function on $[0,\infty)$, and $f'$ to denote a function
whose antiderivative is $f$.

\[
\mathbb{E}\left[f(Y)\right]=f(0)+\int_{0}^{\infty}f'(t)\mathrm{Pr}(Y\geq t)dt
\]
\end{lem}

\begin{proof}

\begin{align*}
&f(0)+\int_{0}^{\infty}f'(t)\mathrm{Pr}(Y\geq t)dt \\ &=f(0)+\int_{0}^{\infty}\int_{t}^{\infty}f'(t)p(y)dydt\\
 & =f(0)+\int_{\{y\geq t,t\geq0\}}f'(t)p(y)d\left(\begin{array}{c}
y\\
t
\end{array}\right)\\
 & =f(0)+\int_{0}^{\infty}\int_{0}^{y}f'(t)p(y)dtdy\\
 & =f(0)+\int_{0}^{\infty}\left(\int_{0}^{y}f'(t)dt\right)p(y)dy\\
 & =f(0)+\int_{0}^{\infty}\left(f(y)-f(0)\right)p(y)dy\quad\text{(2nd FTC)}\\
 & =f(0)+\int_{0}^{\infty}f(y)p(y)dy-\int_{0}^{\infty}f(0)p(y)dy\\
 & =f(0)+\mathbb{E}\left[f(Y)\right]-f(0)\\
 & =\mathbb{E}\left[f(Y)\right]
\end{align*}
\end{proof}

\begin{lem}
With the chosen $\Phi_{\delta}(\cdot)$ and $\Phi_{\delta}'(\cdot)$ shown previously,
$\mathbb{E}_{\{y_{i,j} \sim P_{\phi,i}\}_{i,j}}\left[\Phi_{\delta_{i}}(d(y_{i, j},D_i))\right]=\delta_{i}-\int_{0}^{\delta_{i}}\mathrm{Pr}(d(y_{i, j},D_i)<t)dt$.
\end{lem}

\begin{proof}
By definition, $\Phi_{\delta_{i}}(0)=0$.

\begin{align*}
& \mathbb{E}_{\{y_{i,j} \sim P_{\phi,i}\}_{i,j}}\left[\Phi_{\delta_{i}}(d(y_{i, j},D_i)\right] \\
&=\Phi_{\delta_{i}}(0) + \int_{0}^{\infty}\Phi_{\delta_{i}}'(t)\mathrm{Pr}(d(y_{i, j},D_i)\geq t)dt\quad\\
&\text{(From Lemma 1)}\\
& =\int_{0}^{\delta_{i}} 1 \cdot \mathrm{Pr}(d(y_{i, j},D_i)\geq t)dt \\
&+ \int_{\delta_{i}}^{\infty} 0 \cdot \mathrm{Pr}(d(y_{i, j},D_i)\geq t)dt\\
 & =\int_{0}^{\delta_{i}}\mathrm{Pr}(d(y_{i, j},D_i)\geq t)dt\\
 & =\int_{0}^{\delta_{i}}\left(1-\mathrm{Pr}(d(y_{i, j},D_i)<t)\right)dt\\
 & =\delta_{i}-\int_{0}^{\delta_{i}}\mathrm{Pr}(d(y_{i, j},D_i)<t)dt
\end{align*}
\end{proof}

\begin{lem}
The likelihood above is equivalent to  $\mathcal{L}_{\{\delta_{i}\}_{i}}(\phi)=\frac{1}{n}\sum_{i=1}^{n}\frac{1}{Mw_{i}}\sum_{j=1}^{M}\int_{0}^{\delta_{i}}\mathrm{Pr}(d(y_{i, j},D_i)<t)dt$.
\end{lem}

\begin{proof}

\begin{align*}
&\mathcal{L}_{\{\delta_{i}\}_{i}}(\phi) \\
& =\mathbb{E}_{\{y_{i,j} \sim P_{\phi,i}\}_{i,j}}\left[\frac{1}{n}\sum_{i=1}^{n}\frac{1}{w_{i}}\left(\delta_{i}-\frac{1}{M}\sum_{j=1}^{M}\Phi_{\delta_{i}}(d(y_{i, j},D_i))\right)\right]\\
 & =\frac{1}{n}\sum_{i=1}^{n}\frac{1}{w_{i}}\left(\delta_{i}-\frac{1}{M}\sum_{j=1}^{M}\mathbb{E}_{\{y_{i,j} \sim P_{\phi,i}\}_{i,j}}\left[\Phi_{\delta_{i}}(d(y_{i, j},D_i)\right]\right)\\
 & =\frac{1}{n}\sum_{i=1}^{n}\frac{1}{w_{i}}\\
 &\left(\delta_{i}-\frac{1}{M}\sum_{j=1}^{M}\left(\delta_{i}-\int_{0}^{\delta_{i}}\mathrm{Pr}(d(y_{i, j},D_i)<t)dt\right)\right)\quad\\
 &\text{(From Lemma 2)}\\
 & =\frac{1}{n}\sum_{i=1}^{n}\frac{1}{Mw_{i}}\sum_{j=1}^{M}\int_{0}^{\delta_{i}}\mathrm{Pr}(d(y_{i, j},D_i)<t)dt
\end{align*}
\end{proof}

\subsubsection{Equivalence to loss function for training}
\label{sec:equivalence_loss_function}
Here shows the equivalence of a tractable sample-based loss function used for training. \\

\noindent\textbf{Radius Threshold $\delta_i$}
Lemma 3 shows that the likelihood computes the probability the model $P_phi$ assigns to the neighborhood of the training sample, which is controlled by the radius threshold $\delta_i$. To maximize the likelihood, a small neighborhood is desired, hence a small value of $\delta_i$ is desirable. However, if $\delta_i$ is ``too small", then by the chosen $\Phi_{\delta_i}$, if $d(y_{i, j},D_i)>\delta_i$, for all $j$, then $d(y_{i, j},D_i)=\delta_i \forall j$, which leads to $\frac{1}{n}\sum_{i=1}^{n}\frac{1}{w_{i}}(\delta_{i}-\frac{1}{M}\sum_{j=1}^{M}\Phi_{\delta_{i}}(d(y_{i, j},D_i) =0$. This leads to having no gradients w.r.t. to $\phi$ since it is constant, which does not allow for network training. Thus the smallest $\delta_i$ that can have such that the expression's value is not constant and allows for gradients is $\min_{j\in[M]}d(y_{i, j},D_i)$. The likelihood objective then becomes:

\begin{align*}
& \mathcal{L}_{\{\delta_{i}\}_{i}}(\theta) =\mathbb{E}_{\{y_{i,j} \sim P_{\phi,i}\}_{i,j}}\\
&\left[\frac{1}{n}\sum_{i=1}^{n}\frac{1}{w_{i}}\left(\delta_{i}-\frac{M-1}{M}\delta_{i}-\frac{1}{M}\min_{j\in[M]}d(y_{i, j}, D_i)\right)\right]\\
 & =\mathbb{E}_{\{y_{i,j} \sim P_{\phi,i}\}_{i,j}}\left[\frac{1}{n}\sum_{i=1}^{n}\frac{1}{w_{i}}\left(\frac{1}{M}\delta_{i}-\frac{1}{M}\min_{j\in[M]}d(y_{i, j}, D_i)\right)\right]\\
 & =\mathbb{E}_{\{y_{i,j} \sim P_{\phi,i}\}_{i,j}}\left[\frac{1}{nM}\sum_{i=1}^{n}\frac{1}{w_{i}}\left(\delta_{i}-\min_{j\in[m]}d(y_{i, j}, D_i)\right)\right]
\end{align*}

The sample-based loss function then becomes equivalent to the objective of maximizing the likelihood as follows:

\begin{align*}
&\arg\max_{\phi}\mathcal{L}_{\{\phi_{i}\}_{i}}(\phi) \\
& =\arg\max_{\phi}\mathbb{E}_{\{y_{i,j} \sim P_{\phi,i}\}_{i,j}}\\
&\left[\frac{1}{nM}\sum_{i=1}^{n}\frac{1}{w_{i}}\left(\delta_{i}-\min_{j\in[M]}d(y_{i, j}, D_i)\right)\right]\\
& =\arg\max_{\phi}\mathbb{E}_{\{y_{i,j} \sim P_{\phi,i}\}_{i,j}}\left[\sum_{i=1}^{n}\frac{\delta_i}{w_{i}} - \frac{1}{w_{i}}\min_{j\in[M]}d(y_{i, j}, D_i)\right]\\
& =\arg\max_{\phi}\mathbb{E}_{\{y_{i,j} \sim P_{\phi,i}\}_{i,j}}\left[-\sum_{i=1}^{n} \frac{1}{w_{i}}\min_{j\in[M]}d(y_{i, j}, D_i)\right]\\
 & =\arg\min_{\phi}\mathbb{E}_{\{y_{i,j} \sim P_{\phi,i}\}_{i,j}}\left[\sum_{i=1}^{n}\frac{1}{w_{i}}\min_{j\in[M]}d(y_{i, j}, D_i)\right]
\end{align*}

\noindent which is the sample-based loss function, \ie taking the minimum loss for the set of drawn samples. In our case, $w_i=0 \forall i$, and for each training data point, we sample $M=20$ estimates by drawing $z_j\sim \mathcal{N}(0, \mathbf{I})$, and taking the minimum loss w.r.t. to the corresponding single ground truth depth map $D_i$ for the training data point. This allows us to learn multimodal depth distributions to capture the inherent ambiguities in monocular depth estimation.

\subsubsection{Equivalence to maximizing the sum of the likelihoods.}
\label{sec:imle_convergence_proof}
For completeness, here shows the equivalence of the objective function to maximizing the sum of the likelihood as proven in IMLE~\cite{li2018implicit}. The learning of $\phi$ involves solving a sequence of optimization problems at current values for $\delta_i$, and as optimization progresses later into the sequence, $\delta_i$ becomes smaller and smaller and eventually converges to the maximum likelihood.

\begin{lem}
$\lim_{\{\delta_{i}\to0^{+}\}_{i}}\mathcal{L}_{\{\delta_{i}\}_{i}}(\phi)=\frac{1}{n}\sum_{i=1}^{n}p_{\delta}(D_{i})$.
\end{lem}

\begin{proof}

\begin{align*}
&\mathcal{L}_{\{\delta_{i}\}_{i}}(\phi) \\
& =\frac{1}{n}\sum_{i=1}^{n}\frac{1}{Mw_{i}}\sum_{j=1}^{M}\int_{0}^{\tau_{i}}\mathrm{Pr}(d(y_{i, j}, D_i)<t)dt\quad\\
&\text{(From Lemma 3)}\\
 & =\frac{1}{nM}\sum_{i=1}^{n}\sum_{j=1}^{M}\frac{1}{w_{i}}\int_{0}^{\tau_{i}}\int_{B_{t}(D_i)}p_{\phi, i}(\mathbf{y})d\mathbf{y}dt\\
 & =\frac{1}{nM}\sum_{i=1}^{n}\sum_{j=1}^{M}\frac{\int_{0}^{\delta_{i}}\int_{B_{t}(D_i)}P_{\phi, i}(\mathbf{y})d\mathbf{y}dt}{\int_{0}^{\delta_{i}}\int_{B_{t}(D_{i})}d\mathbf{y}dt}
\end{align*}

\begin{align*}
&\lim_{\{\delta_{i}\to0^{+}\}_{i}}\mathcal{L}_{\{\delta_{i}\}_{i}}(\phi) \\
& =\frac{1}{nM}\sum_{i=1}^{n}\left(\lim_{\delta_{i}\to0^{+}}\left(\sum_{j=1}^{M}\frac{\int_{0}^{\delta_{i}}\int_{B_{t}(D_{i})}p_{\phi, i}(\mathbf{y})d\mathbf{y}dt}{\int_{0}^{\delta_{i}}\int_{B_{t}(D_{i})}d\mathbf{y}dt}\right)\right)\\
 & =\frac{1}{nM}\sum_{i=1}^{n}\sum_{j=1}^{M}\left(\lim_{\delta_{i}\to0^{+}}\frac{\int_{0}^{\delta_{i}}\int_{B_{t}(D_{i})}p_{\phi, i}(\mathbf{y})d\mathbf{y}dt}{\int_{0}^{\delta_{i}}\int_{B_{t}(D_{i})}d\mathbf{y}dt}\right)\\
 & =\frac{1}{nM}\sum_{i=1}^{n}\sum_{j=1}^{M}\left(\lim_{\delta_{i}\to0^{+}}\frac{\int_{B_{\delta_{i}}(D_{i})}p_{\phi, i}(\mathbf{y})d\mathbf{y}}{\int_{B_{\delta_{i}}(D_{i})}d\mathbf{y}}\right)\quad\\
 &\text{(L'H\^{o}pital and 2nd FTC)}\\
 & =\frac{1}{nM}\sum_{i=1}^{n}\sum_{j=1}^{M}\left(\lim_{\delta_{i}\to0^{+}}\frac{\int_{0}^{\delta_{i}}\int_{\{\mathbf{y}\vert d(\mathbf{y},D_i)=r\}}p_{\phi, i}(\mathbf{y})d\mathbf{y}dr}{\int_{0}^{\delta_{i}}\int_{\{\mathbf{y}\vert d(\mathbf{y},D_i)=r\}}d\mathbf{y}dr}\right)\\
 & =\frac{1}{nM}\sum_{i=1}^{n}\sum_{j=1}^{M}\left(\lim_{\delta_{i}\to0^{+}}\frac{\int_{\{\mathbf{y}\vert d(\mathbf{y},D_i)=\delta_{i}\}}p_{\phi, i}(\mathbf{y})d\mathbf{y}}{\int_{\{\mathbf{y}\vert d(\mathbf{y},D_i)=\delta_{i}\}}d\mathbf{y}}\right)\quad\\
 &\text{(L'H\^{o}pital and 2nd FTC)}\\
 & =\frac{1}{nM}\sum_{i=1}^{n}\sum_{j=1}^{M}p_{\phi, i}(D_i)\\
 & =\frac{1}{n}\sum_{i=1}^{n}p_{\phi, i}(D_i)
\end{align*}
\end{proof}



\section{Training Images Samples}
We also show samples of train images from the three scenes in each of the three datasets in our experiments. Samples are shown in Figure~\ref{fig:train_images}.

\begin{figure*}
  \centering
    \includegraphics[width=\textwidth]{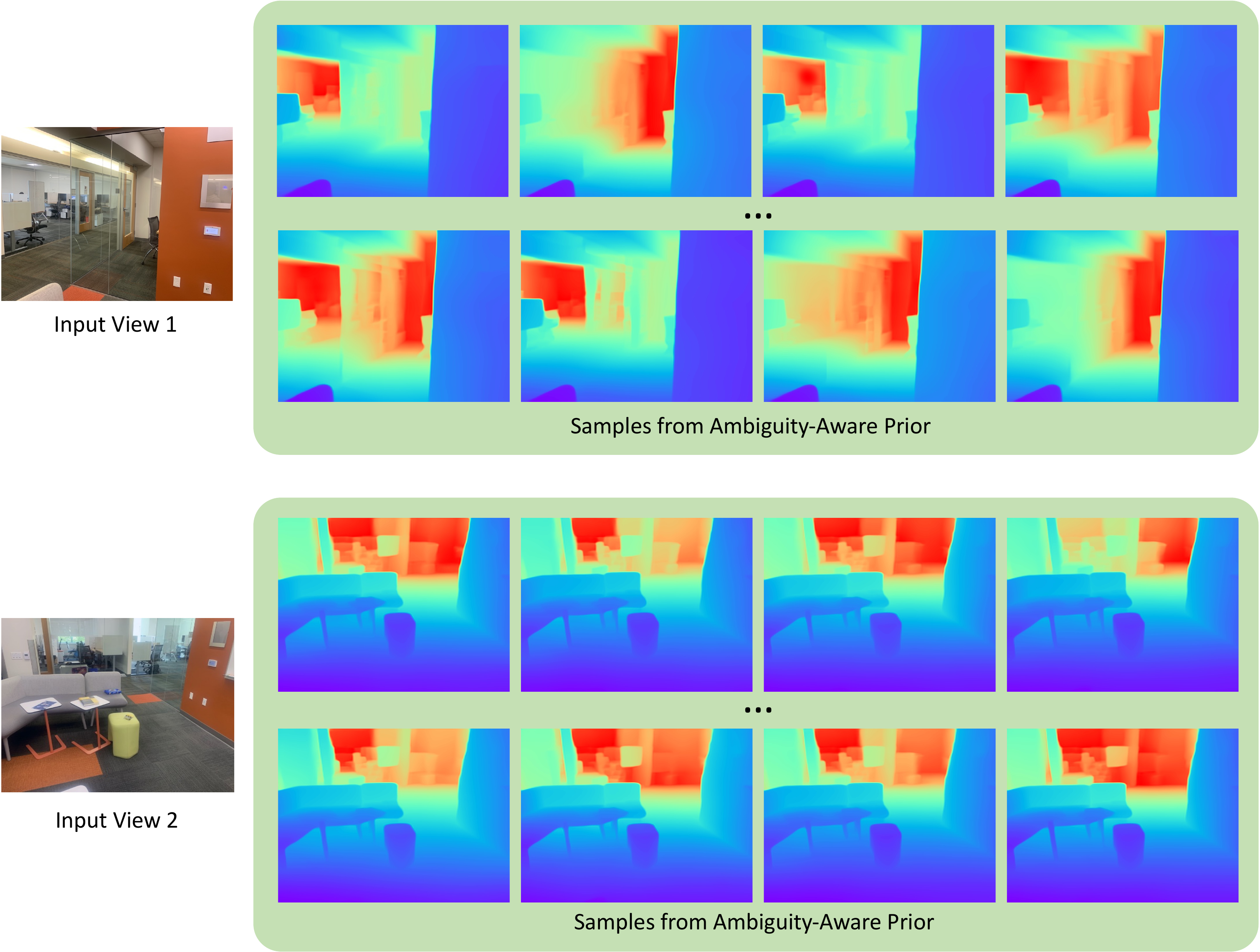}
    \captionof{figure}{\textbf{Ambiguity-Aware Depth Estimates}. Hypothesis from two input views with non-opaque surfaces. This figure shows that in both cases, our ambiguity-aware prior is able to recover a distribution of depth estimates that is multimodal. These multimodal distributions allow to capture the room as well as the recover the objects behind the glass. Given the multiple views with multimodal distributions, NeRF is able to \textbf{find the mode} that is consistent, hence allowing for less blurry and better photometric reconstruction. Please view the attached video demo for the results on this In-the-Wild scene.
    }`
    \label{fig:teaser_supp}
\end{figure*}

\begin{figure*}[t]
    \centering
    \includegraphics[width=0.91\linewidth]{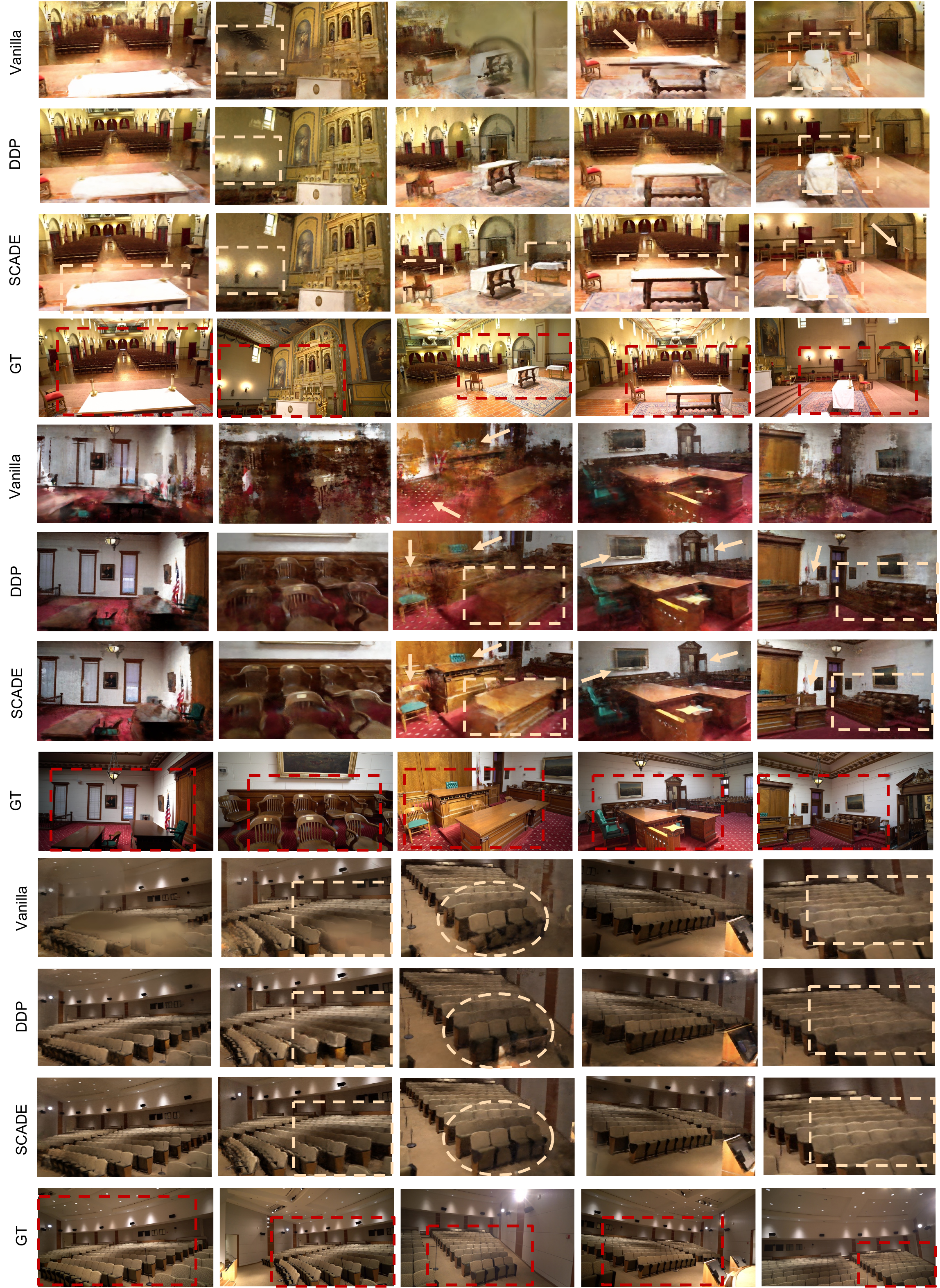}
    \caption{Qualitative Results for the Tanks and Temples~\cite{Knapitsch2017} dataset.}
    \label{fig:tanks_and_temples_quali}
\end{figure*}

\begin{figure*}[t]
    \centering
    \includegraphics[width=\linewidth]{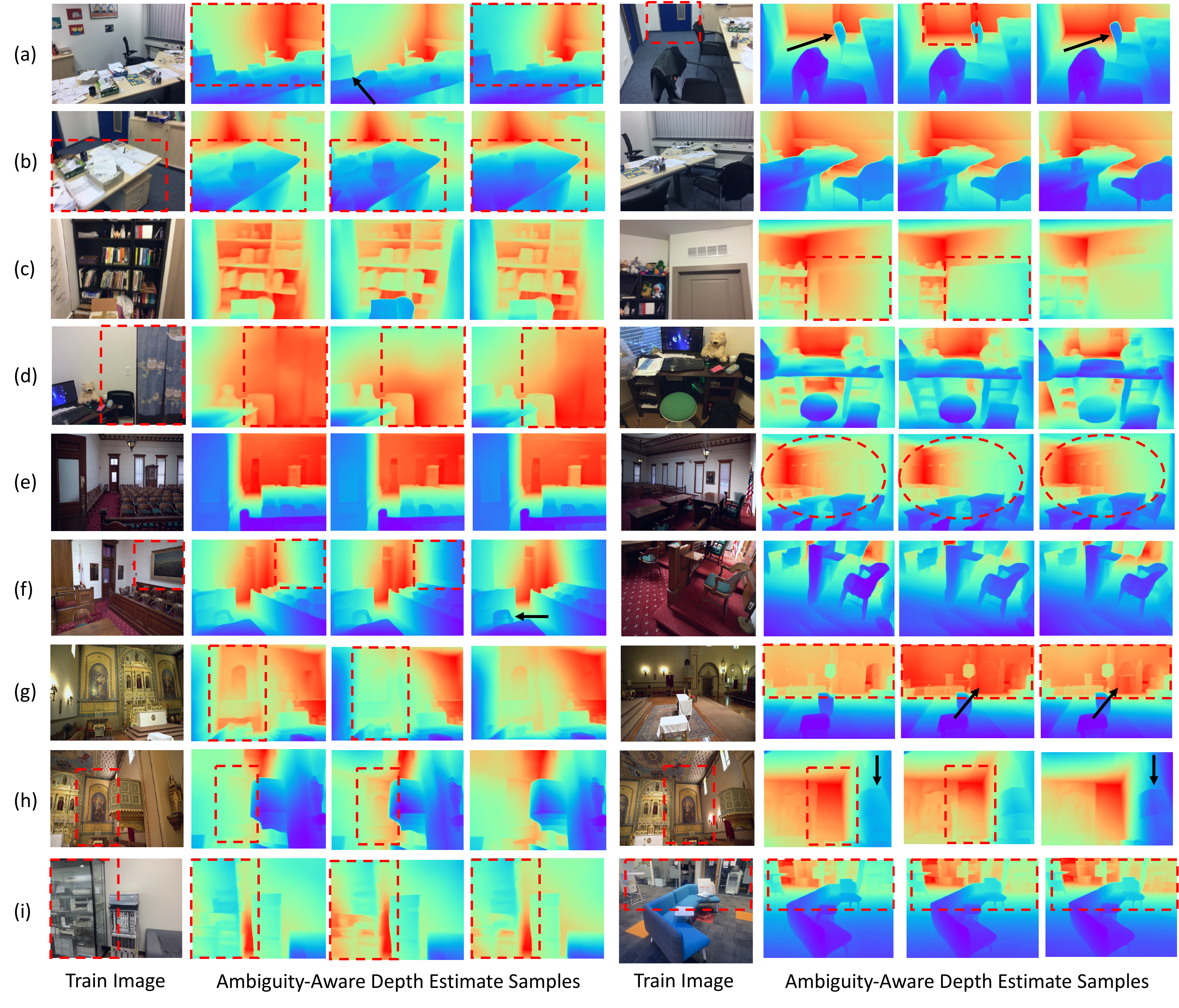}
    \caption{\textbf{Samples from our Ambiguity-Aware Depth Estimates on train images of the different scenes used in our experiments}. Ambiguity is shown in [Left; right]: (a) How far the back wall is relative to the chair as well as the width of the cabinet and how far it is relative to the desk; whether the door is at a different compared to the wall and the relative depth of the the second chair w.r.t. to the nearer chair and the wall. (b) Objects on the desk have varying depths, e.g. it is unclear from a single view whether the papers have a thickness or not; relative depth of the chair w.r.t. the wall and the camera (c) Depth of the bookshelf; albedo v.s. shading of the door w.r.t to the door frame. (d) Depth of the curtain, whether it is flat on the wall or not, and without scene context, it can also be interpreted as painted texture on the wall; relative depths of the different cluttered objects. (e) Relative depths of the barrier, the seats and the far back wall with a cabinet; depth of the far back corner of the room w.r.t. the desk and chair and the camera. (f) Whether the painting is flat on the wall or the frame protrudes it out; relative depths of the chairs and the far back wall. (g) Whether the painted texture is convex or is flat (i.e. just painted) on the wall; whether there is a far back door or is just a texture on the wall. (h) Both are similar to g. left but on different viewpoints and on the opposite side of the room. (i) Non-opaque surface ambiguity due to the glass cabinet; glass door behind the sofa is also non-opaque. 
    }
    \label{fig:depth_hypothesis_many}
\end{figure*}

\begin{figure*}[t]
    \centering
    \includegraphics[width=0.9\linewidth]{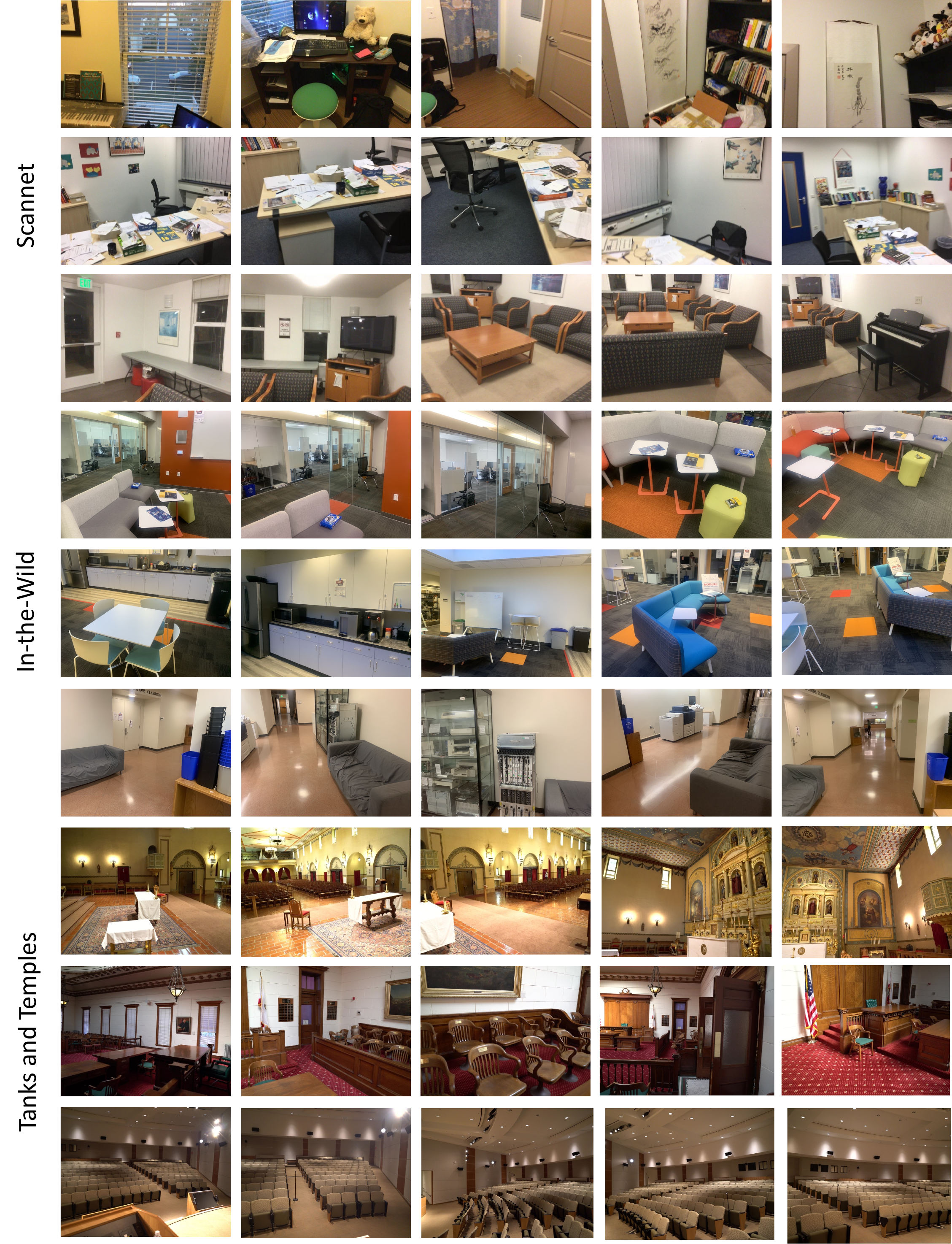}
    \caption{Samples of training images from the three scenes from the three datasets - Scannet~\cite{dai2017scannet}, In-the-Wild and Tanks and Temples~\cite{Knapitsch2017}.}
    \label{fig:train_images}
\end{figure*}

\end{document}